\newif\ifarxiv

\arxivfalse

\ifarxiv
  \documentclass[10pt, DIV=12, enabledeprecatedfontcommands]{scrartcl}
  \usepackage{authblk}
  \usepackage{etoolbox}
  \usepackage{natbib}
  \usepackage{mathpazo}
  \setkomafont{sectioning}{\bfseries}
  \setkomafont{author}{\small}

  \makeatletter
  \patchcmd{\@maketitle}{\titlefont\huge}{\titlefont\small}{}{}
  \makeatother
\else
  \documentclass{article} 
 
  \PassOptionsToPackage{sort}{natbib}
 
  \usepackage{iclr2022_conference,times}
  
\fi

\usepackage{amsmath}        
\usepackage{amsthm}         
\usepackage{dsfont}         
\usepackage[utf8]{inputenc} 
\usepackage[T1]{fontenc}    
\usepackage[hidelinks]{hyperref}       
\usepackage{url}            
\usepackage{booktabs}       
\usepackage{amsfonts}       
\usepackage{listings}       
\usepackage{nicefrac}       
\usepackage{microtype}      
\usepackage{mleftright}     
\usepackage{paralist}       
\usepackage{xspace}         
\usepackage{caption}        
\usepackage{subcaption}     
\usepackage{graphicx}       
\usepackage{siunitx}        
\usepackage{tikz}           
\usepackage{pgfplots}       
\usepackage{enumitem}       
\usepackage{xcolor}         
\usepackage{pgfplotstable}
\usepackage{censor}         
\usepackage{multirow}       
\usepackage{wrapfig}        
\usepackage[final]{changes} 

\usetikzlibrary{arrows}
\usetikzlibrary{calc}
\usetikzlibrary{chains}
\usetikzlibrary{fit}
\usetikzlibrary{matrix}
\usetikzlibrary{positioning}
\usetikzlibrary{scopes}
\usetikzlibrary{external}

\pgfplotsset{compat=1.16}
\usepgfplotslibrary{colorbrewer}
\usepgfplotslibrary{groupplots}
\usepgfplotslibrary{colormaps}

\pgfplotsset{
  compat                 = 1.16,  
  filter discard warning = false, 
  discard if not/.style 2 args={
    x filter/.append code={
      \edef\tempa{\thisrow{#1}}
      \edef\tempb{#2}
      \ifx\tempa%
        \tempb%
      \else%
        
      \fi
    }
  },
}

\pgfplotsset{%
  every non boxed x axis/.append style={x axis line style=-},
  every non boxed y axis/.append style={y axis line style=-}
}

\definecolor{538blue}   {RGB}{0,143,213}
\definecolor{538red}    {RGB}{252,79,48}
\definecolor{538yellow} {RGB}{229,174,56}
\definecolor{538green}  {RGB}{109,144,79}
\definecolor{538gray}   {RGB}{139,139,139}
\definecolor{538purple} {RGB}{129,15,124}

\newtheorem{theorem}   {Theorem}

\DeclareMathOperator{\CAT}        {CAT}
\DeclareMathOperator{\degree}     {deg}
\DeclareMathOperator{\distance}   {d}

\DeclareMathOperator{\kernel}     {k}
\DeclareMathOperator{\MMD}        {MMD}

\newcommand{\edges}    {\ensuremath{E}\xspace}                            
\newcommand{\graph}    {\ensuremath{\mathrm{G}}\xspace}                   
\newcommand{\graphs}   {\ensuremath{\mathcal{G}}\xspace}                  
\newcommand{\laplacian}{\ensuremath{\mathcal{L}}\xspace}                  
\newcommand{\metric}[1]{\ensuremath{\distance_{\mathrm{#1}}}}             
\newcommand{\model}    {\ensuremath{\mathcal{M}}\xspace}                  
\newcommand{\reals}    {\ensuremath{\mathds{R}}\xspace}                   
\newcommand{\vertices} {\ensuremath{V}\xspace}                            

\ifarxiv
  \title{\LARGE Evaluation Metrics for Graph Generative Models: Problems, Pitfalls, and Practical Solutions}
  \date{}

  \author[1, 2, $\ast$]{Leslie O'Bray}
  \author[1, 2, $\ast$]{Max Horn}
  \author[1, 2, 3, $\dagger$]{Bastian~Rieck}
  \author[1, 2, $\dagger$]{Karsten Borgwardt}
  \affil[1]{Department of Biosystems Science and Engineering, ETH
  Zurich, 4058 Basel, Switzerland}
  \affil[2]{SIB Swiss Institute of Bioinformatics, Switzerland}
  \affil[3]{B.R.\ is now with the Institute of AI for Health, Helmholtz Zentrum München, Neuherberg, Germany}

  \affil[$\ast$]{These authors contributed equally.}
  \affil[$\dagger$]{These authors jointly supervised this work.}

\else
  \title{Evaluation Metrics for Graph Generative Models:  Problems, Pitfalls, and Practical \\ Solutions}
    
  \author{Leslie O'Bray$^{1,2,\dagger}$, Max Horn$^{1,2,\dagger}$, Bastian Rieck$^{1,2,3,4,*}$ and Karsten Borgwardt$^{1,2,*}$ \\
  $^1$Department of Biosystems Science and Engineering, ETH Zürich, Switzerland \\
  $^2$SIB Swiss Institute of Bioinformatics, Switzerland \\
  $^3$Institute of AI for Health, Helmholtz Munich, Germany\\
  $^4$Technical University of Munich, Germany\\
  $^{\dagger}$ These authors contributed equally. $^*$ These authors jointly supervised this work. \\
  }
\fi

\makeatletter
\@namedef{Changes@AuthorColor}{cyan!60!black}
\colorlet{Changes@Color}{cyan!60!black}
\makeatother

\iclrfinalcopy 
\begin{document}

\maketitle

\begin{abstract}
Graph generative models are a highly active branch of machine learning. Given the steady development of new models of ever-increasing complexity, it is necessary to provide a principled way to \emph{evaluate} and \emph{compare} them. In this paper, we enumerate the desirable criteria for such a comparison metric and provide an overview of the status quo of graph generative model comparison in use today, which predominantly relies on the maximum mean discrepancy (MMD). We perform a systematic evaluation of MMD in the context of graph generative model comparison, highlighting some of the challenges and
 pitfalls researchers inadvertently may encounter. After conducting a thorough analysis of the behaviour of MMD on synthetically-generated perturbed graphs as well as on recently-proposed graph generative models, we are able to provide a suitable procedure to mitigate these challenges and pitfalls. We aggregate our findings into a list of practical recommendations for researchers to use when evaluating graph generative models.
\end{abstract}

\section{Introduction}

Graph generative models have become an active research branch, making it
possible to generalise structural patterns inherent to graphs from
certain domains---such as chemoinformatics---and actively synthesise
\emph{new} graphs~\citep{Liao19}.
Next to the development of improved models, their \emph{evaluation} is
crucial. This is a well-studied issue in other domains, leading to
metrics such as the `Fr{\'e}chet Inception Distance'~\citep{Heusel17}
for comparing image-based generative models. Graphs, however, pose their
own challenges, foremost among them being that an evaluation based on
visualisations, i.e.\ on \emph{perceived} differences, is often not
possible. In addition, virtually all relevant structural properties of
graphs exhibit spatial invariances, e.g., connected components and cycles are invariant with respect to rotations---that
have to be taken into account by a comparison metric.

While the community has largely gravitated towards a single comparison metric, the maximum mean discrepancy (MMD) \added{\citep{You18, Liao19, Niu2020, Goyal2020, Dai2020, Zhang2021, Chen2021, Mi2021, podda2021graphgenredux}}, neither its expressive power nor its other properties have been systematically investigated in the context of graph generative model comparison. The goal of this paper is to provide such an investigation, starting from first principles by describing the desired properties of such a comparison metric, providing an overview of what is done in practice today, and systemically assessing MMD's behaviour using recent graph generative models. We highlight some of the caveats and shortcomings of the existing status quo, and provide researchers with practical recommendations to address these issues. \added{We note here that our investigations focus on assessing the structural similarity of graphs, which is a necessary first step before one can jointly assess graphs based on structural and attribute similarity.}
This paper purposefully refrains from developing its \emph{own} graph
generative model to avoid any bias in the comparison.

\newcommand\gauss[2]{1/(#2*sqrt(2*pi))*exp(-((x-#1)^2)/(2*#2^2))} 
\begin{figure}
  \centering
  \tikzexternaldisable
  \newsavebox\tikzrepresentation
  \savebox\tikzrepresentation{%
    \pgfmathsetseed{42}
    \begin{tikzpicture}
      \begin{axis}[
        axis equal,
        axis lines* = middle,
        xtick       = \empty,
        ytick       = \empty,
        scale       = 0.5,
        xmin        = -6,
        ymin        = -6,
        ymax        =  6,
        xmax        =  6,
      ]
        \addplot[%
          only marks,
          samples   = 50,
          mark      = *,
          mark size = 1.0pt,
        ] {25 * rand};
        \node at (-5, 5) {$\reals^d$};
      \end{axis}
    \end{tikzpicture}
  }
  \newsavebox\mmdbox
  \savebox\mmdbox{%
    \begin{tikzpicture}
        \begin{axis}[
            every axis plot post/.append style={
            samples=50,smooth
                    }, 
            height=3cm,
            width=5cm,
            axis x line*=bottom, 
            axis y line*=left, 
            enlargelimits=upper,
            yticklabels={,,},
            xticklabels={,,}
                ] 
            \addplot [fill=purple!40!gray, draw=purple!40!gray, opacity=0.5]{\gauss{0}{1.5}};
            \addplot[fill=cyan!40!gray,draw=cyan!40!gray, opacity=0.5] {\gauss{2}{0.75}};
            \end{axis}
        \end{tikzpicture}
        }
  \newsavebox\tikzgraphs
  \savebox\tikzgraphs{%
    \begin{tikzpicture}
      \tikzset{
        every node/.style = {
          shape        = circle,
          minimum size = 3pt,
          inner sep    = 0pt,
          fill         = black,
          draw         = black,
          font         = \small,
        },
      }
      \begin{scope}[scale=0.5]
        \node (N0) at (-0.58, 0.25) {};
        \node (N1) at (-0.17, 0.04) {};
        \node (N2) at (0.23, -0.21) {};
        \node (N3) at (0.61, 0.08) {};
        \node (N4) at (-0.68, 0.69) {};
        \node (N5) at (-1.00, 0.07) {};
        \node (N6) at (0.05, -0.65) {};
        \node (N7) at (0.53, -0.55) {};
        \node (N8) at (1.00, 0.28) {};
        \draw (N0) -- (N1);
        \draw (N0) -- (N4);
        \draw (N0) -- (N5);
        \draw (N1) -- (N2);
        \draw (N2) -- (N3);
        \draw (N2) -- (N6);
        \draw (N2) -- (N7);
        \draw (N3) -- (N8);
      \end{scope}

      \begin{scope}[scale=0.25, yshift=-3.5cm]
        \node (N0) at (-0.00, 0.00) {};
        \node (N1) at (1.00, 1.00) {};
        \node (N2) at (-1.00, 1.00) {};
        \node (N3) at (-1.00, -1.00) {};
        \node (N4) at (1.00, -1.00) {};
        \draw (N0) -- (N1);
        \draw (N0) -- (N2);
        \draw (N0) -- (N3);
        \draw (N0) -- (N4);
      \end{scope}

      \begin{scope}[scale=0.50, yshift=-4.0cm]
        \node (N0) at (0.05, -0.23) {};
        \node (N1) at (-0.08, 0.40) {};
        \node (N2) at (-0.21, 1.00) {};
        \node (N3) at (-0.41, -0.70) {};
        \node (N4) at (0.66, -0.47) {};
        \draw (N0) -- (N1);
        \draw (N0) -- (N3);
        \draw (N0) -- (N4);
        \draw (N1) -- (N2);
      \end{scope}
    \end{tikzpicture}
  }
  \newsavebox\tikzhist
  \savebox\tikzhist{%
    \pgfmathsetseed{42}
    \begin{tikzpicture}
      \begin{axis}[
        axis x line = none,
        axis y line = none,
        xtick       = \empty,
        ytick       = \empty,
        scale       = 0.25,
        height      =  5.0cm,
        width       = 10.0cm,
      ]
        \addplot[%
          hist,
          hist/bins = 10,
          samples   = 500,
          draw      = white,
          fill      = black,
          ] {rnd};
      \end{axis}
    \end{tikzpicture}

  }
  \tikzset{%
    block/.style = {%
      draw,
      align          = center,
      draw           = black,
      minimum height = 3.50cm,
      text width     = 2.50cm,
      inner sep      = 5pt,
      shape          = rectangle,
      font           = \scriptsize,
    },
    line/.style = {%
      >=stealth',
      draw,
      ->,
    }%
  }
  \begin{tikzpicture}[%
    start chain   = going right,
    node distance =5.0mm,
  ]
    \node[block, on chain, text width = 1.25cm, label = below:{\small Graphs}] (Graphs) {
      \usebox\tikzgraphs
    };

    \node[%
      block,
      on chain,
      align          = left,
      right          = of Graphs.south east,
      anchor         = south west,
      label          = below:{\small Descriptor functions}
    ] (Descriptor functions) {%
      Clustering coefficient\\[0.25cm]
      Degree distribution\\[0.25cm]
      Laplacian spectrum
      \begin{center}
        \usebox\tikzhist
      \end{center}
    };

    \node[%
      block,
      on chain,
      text width     = 3.50cm,
      align          = left,
      right          = of Descriptor functions.south east,
      anchor         = south west,
      label          = below:{\small Representations}
    ] (Representations) {%
      \usebox\tikzrepresentation
    };

    \node[%
      block,
      on chain,
      text width     = 3.75cm,
      align          = left,
      right          = of Representations.south east,
      anchor         = south west,
      label          = below:{\small Evaluator function}
    ] (MMD) {%
    
    \begin{center}
    \usebox\mmdbox
    \\
    MMD
    \end{center}
    };

    \draw[line] (Graphs) -- (Descriptor functions);
    \draw[line] (Descriptor functions) -- (Representations);
    \draw[line, <-] (MMD.west) -- (MMD.west -| Representations.east);
  \end{tikzpicture}
  \caption{%
    An overview of the workflow used to evaluate graph generative
    models, \added{as is used, e.g., in \citet{You18, Liao19, Niu2020}}: given a distribution of graphs, a set of descriptor
    functions is employed to map each graph to a high-dimensional
    representation in~$\reals^d$. These representations are then
    compared~(with a reference distribution or with each other) using an 
    evaluator function called the maximum mean
    discrepancy~(MMD). In principle, MMD does not require the vectorial
    representation, but we find that this is the predominant use in the
    context of graph generative model evaluation.
  }
  \label{fig:Overview}
\end{figure}
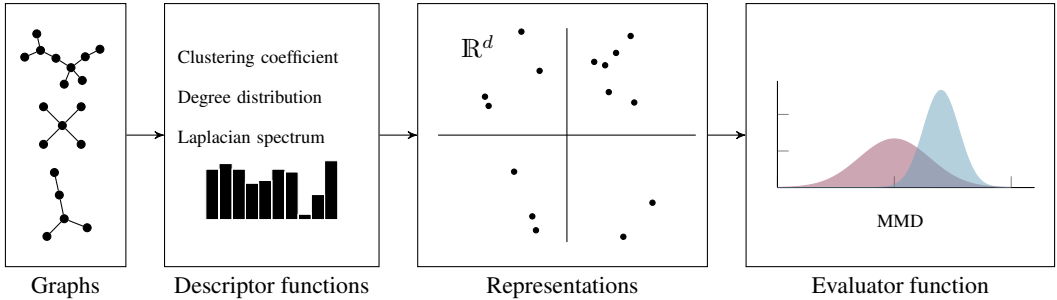

\section{Comparing graph distributions}\label{sec:Desiderata}

In the following, we will deal with undirected graphs. We denote
such a graph as~$\graph = (\vertices, \edges)$ with vertices~$\vertices$
and edges~$\edges$. We treat graph generative models as black-box models, each of
which results in a set\footnote{%
  Formally, the output can also be a multiset, because graphs are
  allowed to be duplicated.
} of graphs~$\graphs$. The original empirical distribution of graphs is denoted as
$\graphs^\ast$. Given models $\{\model_1, \model_2, \dots\}$, with generated
sets of graphs $\{\graphs_1, \graphs_2, \dots\}$, the goal
of generative model evaluation is to assess which model is a better fit,
i.e.\ which distribution is closer to
$\graphs^\ast$.
This requires the use of a (pseudo-)metric
$\metric{}\mleft(\cdot, \cdot\mright)$ to assess the dissimilarity between
$\graphs^\ast$ and generated graphs.
We argue that the \emph{desiderata} of any such comparison metric are as follows:
\begin{enumerate}[leftmargin = 0.5em]
  \item \textbf{Expressivity}: if $\graphs$ and $\graphs'$ do not arise
    from the same distribution, a suitable metric should be able to
    detect this.
    Specifically, $\metric{}\mleft(\graphs, \graphs'\mright)$ should be
    monotonically increasing as $\graphs$ and $\graphs'$ become increasingly dissimilar.
    
  \item \textbf{Robustness}: if a distribution~$\graphs$ is subject to
    a perturbation, a suitable metric should be robust to small
    perturbations. Changes in the metric should be ideally upper-bounded
    by a function of the amplitude of the perturbation. Robust metrics are preferable
    because of the inherent stochasticity of training generative models.

  \item \textbf{Efficiency}: model comparison metrics should
    be reasonably fast to calculate; even though model evaluation is
    a \emph{post hoc} analysis, a metric should scale well
    with an increasing number of graphs and an increasing size of said
    graphs.
\end{enumerate}
While there are many ways to compare two distributions, ranging from
statistical divergence measures to proper metrics in the
mathematical sense, the comparison of
distributions of graphs is exacerbated by the fact that individual
graphs typically differ in their cardinalities and are only described up
to permutation. With distances such as the graph edit distance being
NP-hard in general~\citep{Zheng09}, which precludes using them as an
\emph{efficient} metric, a potential alternative is provided by using
descriptor functions.
A descriptor function~$f$ maps a graph~$\graph$
to an auxiliary representation in some space~$\mathcal{Z}$. The problem
of comparing a generated distribution $\graphs = \{\graph_1, \dots, \graph_n\}$
to the original distribution $\graphs^\ast = \{\graph_1^\ast, \dots,
\graph_m^\ast\}$ thus boils down to comparing the \emph{images}
$f\mleft(\graphs\mright) := \{f\mleft(\graph_1\mright), \dots,
f\mleft(\graph_n\mright)\} \subseteq \mathcal{Z}$ and
$f\mleft(\graphs^\ast\mright) \subseteq \mathcal{Z}$ by any preferred
statistical distance in~$\mathcal{Z}$~(see \autoref{fig:Overview}).

\section{Current state of graph generative model evaluation: MMD}

Of particular interest in
previous literature is the case of $\mathcal{Z} = \reals^d$ and using
the \emph{maximum mean discrepancy}~$\metric{MMD}\mleft(\cdot,
\cdot\mright)$ as a metric.
MMD is one of the most versatile and
expressive options available for comparing distributions of structured
objects such as graphs~\citep{Borgwardt06, Gretton07},
providing also a principled way to perform
two-sample tests~\citep{Bounliphone15, Gretton12, Lloyd15}.
It enables the comparison of two
statistical distributions by means of \emph{kernels}, i.e.\ similarity
measures for structured objects. Letting $\mathcal{X}$ refer to
a non-empty set, a function $\kernel\colon\mathcal{X} \times \mathcal{X} \to \reals$
is a kernel if $\kernel(x_i, x_j) = \kernel(x_j, x_i)$ for $x_i, x_j \in
\mathcal{X}$ and $\sum_{i,j} c_i c_j \kernel(x_i, x_j) \geq 0$ for $x_i,
x_j \in \mathcal{X}$ and $c_i, c_j \in \reals$. MMD uses such a kernel
function to assess the distance between two distributions.
Given~$n$ samples $X = \{x_1, \dots, x_n\} \subseteq
\mathcal{X}$ and~$m$ samples $Y = \{y_1, \dots, y_m\} \subseteq
\mathcal{X}$, the biased empirical estimate of the MMD between~$X$ and~$Y$
is obtained as
\begin{equation}
  \MMD^2(X, Y) := \frac{1}{n^2} \sum_{i,j = 1}^{n} \kernel(x_i, x_j)
  + \frac{1}{m^2} \sum_{i,j = 1}^{m} \kernel(y_i, y_j)
  - \frac{2}{nm} \sum_{i=1}^{n} \sum_{j = 1}^{m} \kernel(x_i, y_j).
  \label{eq:MMD}
\end{equation}
%
Since MMD is known to be a metric on the space of probability
distributions under certain conditions, \autoref{eq:MMD} is often
treated as a metric as well \added{\citep{You18, Liao19, Niu2020}}.\footnote{We will
  follow this convention in this paper and refer to \autoref{eq:MMD} as
  a distance. 
}
%
We use the unbiased empirical estimate of MMD~\citep[Lemma~6]{Gretton12} in our experiments, which removes the self-comparison terms in Equation~\ref{eq:MMD}. MMD has been adopted by the community and the current workflow includes two steps: 
\begin{inparaenum}[(i)]
  \item choosing a descriptor function $f$ as described above, and
  \item choosing a kernel on $\reals^d$ such as an RBF kernel. 
\end{inparaenum}
One then evaluates $\metric{MMD}\mleft(\graphs, \graphs^\ast\mright) := \MMD\mleft(f\mleft(\graphs\mright),
f\mleft(\graphs^\ast\mright)\mright)$ for a sample of graphs~$\graphs$.
Given multiple distributions $\{\graphs_1, \graphs_2, \dots\}$, the values
$\metric{MMD}\mleft(\graphs_i, \graphs^\ast\mright)$ can be used to
\emph{rank} models: smaller values are assumed to indicate a larger
agreement with the original distribution~$\graphs^\ast$. We will now
describe this procedure in more detail and highlight some of its
pitfalls.

\begin{table}[t]
  \caption{The kernels \& parameters chosen by three graph generative models for the MMD calculation.}
  \label{tab:kernel_choice}
  \resizebox{\textwidth}{!}{%
  \centering
  \begin{tabular}{lllccc}
  \toprule
  \multicolumn{1}{c}{\bf Model}  &\multicolumn{2}{c}{\bf Kernel} &\multicolumn{3}{c}{\bf Parameter choice $\sigma$ and $n_\mathrm{bin}$}
  \\ 
   &  & &  Degree & Clustering & Laplacian \\
  \toprule \\
  Model A             & EMD       & $\exp{(\nicefrac{W(x,y)}{2\sigma^2})}$ & $\sigma=1$, $n_\mathrm{bin} = \mathrm{max  degree}$ & $\sigma=0.1$, $n_\mathrm{bin} = 100$ & N/A\\
  Model B             & TV        & $\exp\mleft(-\frac{\metric{TV}(x, y)^2}{2\sigma^2}\mright)$ & $\sigma=1$, $n_\mathrm{bin} = \mathrm{max  degree}$ & $\sigma=0.1$, $n_\mathrm{bin} = 100$ & $\sigma=1$, $n_\mathrm{bin} = 200$ \\
  Model C             & RBF  & $\exp(-\nicefrac{\|x - y\|^2}{2\sigma^2})$ & $\sigma=1$, $n_\mathrm{bin} = \mathrm{max  degree}$ & $\sigma=0.1$, $n_\mathrm{bin} = 100$ & $\sigma=1$, $n_\mathrm{bin} = 200$ \\
  \bottomrule
  \end{tabular}}
\end{table}

\subsection{Kernels \& Descriptor functions}

Before calculating the MMD distance between two samples of graphs, we need
to define both the kernel function $\kernel$ and the
descriptor function~$f$ that will convert a graph $\graph$ to
a representation in~$\mathcal{Z} = \reals^d$, for use in the MMD calculation in
\autoref{eq:MMD}.
We observed a variety of kernel choices in the existing literature. In
fact, in three of the most popular graph generative models in use today,
which we explore in detail in this paper, a different kernel was chosen
for each one. These include a kernel using the first Wasserstein
distance~(EMD), total variation distance~(TV), and the radial basis
function kernel~(RBF), and are listed in Table~\ref{tab:kernel_choice}. 
In the current use of MMD, a descriptor function~$f$ is used to create
a vectorial representation of a graph for use in the kernel computation.
We find that several descriptor functions are commonly employed, either based
on summary statistics of a graph, such as degree distribution histogram
and clustering coefficient histogram, or based on spectral properties of
the graph, such as the Laplacian spectrum histogram. \added{While several papers also consider the orbit as a descriptor function, we do not consider it in depth here due to its computational complexity, which violates the ``efficiency'' property from our desiderata.} We will now provide
brief explanations of these prominent descriptor functions.
%
\paragraph{Degree distribution histogram.}
%
Given a graph $\graph = (\vertices, \edges)$, we obtain a histogram by
evaluating $\degree(v)$ for $v \in \vertices$, where position~$i$ of the
resulting histogram is the number of vertices with degree~$i$. Assuming
a maximum degree~$d$ and extending the histogram with zeros whenever
necessary, we obtain a mapping $f\colon\graphs\to\reals^d$. This
representation has the advantage of being easy to calculate and easy to
compare; by normalising it~(so that it sums to~$1$), we obtain
a size-invariant descriptor.
%
%
\paragraph{Clustering coefficient.}
%
The~(local) clustering coefficient of a vertex~$v$ is defined as the
fraction of edges within its neighbourhood divided by the number of all
possible edges between neighbours, i.e.\
\begin{equation}
  C(v) := \frac{2\mleft|\mleft\{(v_i, v_j) \in E \mid  \added{v_i \in \mathcal{N}(v) \lor v_j \in \mathcal{N}(v)} \mright\}\mright|}{\degree(v) \mleft(\degree(v) - 1\mright)}.
\end{equation}
The value of $C(v) \in [0, 1]$ measures to what extent a vertex~$v$ forms a clique~\citep{Watts98}.
The collection of all clustering coefficients of a graph can be binned
and converted into a histogram in order to obtain a graph-level descriptor. This
function is also easy to calculate but is inherently local; a graph
consisting of disconnected cliques or a fully-connected graph cannot be
distinguished, for example.
\paragraph{Laplacian spectrum histogram.} 
Spectral methods involve assigning a matrix to
a graph~$\graph$, whose spectrum, i.e.\ its eigenvalues and
eigenvectors, is subsequently used as a characterisation of~$\graph$.
Let $\mathbf{A}$ refer to the \emph{adjacency matrix} of~$\graph$, with
$\mathbf{A}_{ij} = 1$ if and only if vertices $v_i$ and $v_j$ are
connected by an edge in~$\graph$~(since~$\graph$ is undirected,
$\mathbf{A}$ is symmetric).
The \emph{normalised graph Laplacian} is defined as
$\laplacian := \mathbf{I} - \mathbf{D}^{-\frac{1}{2}} \mathbf{A} \mathbf{D}^{-\frac{1}{2}}$,
where $\mathbf{I}$ denotes the identity matrix and $\mathbf{D}$ refers
to the \emph{degree matrix}, i.e.\ $\mathbf{D}_{ii} = \deg(v_i)$ for
a vertex~$v_i$ and $\mathbf{D}_{ij} = 0$ for $i \neq j$. The
matrix~$\laplacian$ is real-valued and symmetric, so it is
diagonalisable with  a full set of eigenvalues and eigenvectors.
Letting $\lambda_1 \leq \lambda_2 \leq \dots$ refer to the eigenvalues
of~$\laplacian$, we have $0 \leq \lambda_i \leq 2$~\citep[Chapter~1,
Lemma~1.7]{Chung97}.
This boundedness lends itself naturally to a histogram
representation~(regardless of the size of $\graph$), making it possible
to bin the eigenvalues and use the resulting histogram as a simplified
descriptor of a graph. The expressivity of such a representation is not
clear a priori; the question of whether graphs are fully determined by
their spectrum is still open~\citep{Dam03} and has only been partially
answered \emph{in the negative} for certain classes of
graphs~\citep{Schwenk73}.

\section{Issues with the current practice}\label{subsec:Issues}

Common practice in graph generative model papers is to use MMD by
fixing a kernel and parameter values for the descriptor functions and
kernel, and then assessing the newly-proposed model as well as
its competitor models using said kernel and parameters.
Authors chose
different values of $\sigma$ for the different descriptor functions, but
to the best of our knowledge, all parameters are set to a fixed value
a priori \emph{without} any selection procedure. If MMD were able to give
results and rank different models in a stable way across different
kernel and parameter choices, this would be inconsequential. However,
as our experiments will demonstrate, the results of MMD are \emph{highly
sensitive} to such choices and can therefore lead to an arbitrary
ranking of models. 

Subsequently, we use three current real-world models, GraphRNN~\citep{You18},
GRAN~\citep{Liao19}, and Graph Score Matching~\citep{Niu2020}.  We ran
the models using the author-provided implementations to generate new
graphs on the Community, Barab\'{a}si-Albert, Erd\"{o}s-R\'{e}nyi, and Watts-Strogatz graph datasets, and then calculated the MMD distance between
the generated graphs and the test graphs, using the different 
\begin{inparaenum}[(i)]
  \item kernels that they used~(EMD, TV, RBF),
  \item descriptor functions~(degree histogram, clustering coefficient histogram, and Laplacian spectrum), and
  \item parameter ranges~($\sigma, \lambda \in \{10^{-5}, \ldots,
    10^{5}\}$).
\end{inparaenum}
For simplicity, we will refer to the parameter in all kernels as
$\sigma$. We purposefully \emph{refrained from using model names},
preferring instead `A, B, C' in order to focus on the issues imposed by
such an evaluation, rather than providing a commentary on
the performance of a specific model.

In the following, we will delve deeper into issues originating from the individual components of
the graph generative model comparison in use today. Due to space constraints, examples are provided on individual
datasets; full results across datasets are available in
Appendix~\ref{sec:Full perturbation results}. 

\begin{figure}[tbp]
    \centering
    \begin{subfigure}[b]{0.33\textwidth}
        \resizebox{\textwidth}{!}{%
        \includegraphics[width=\textwidth]{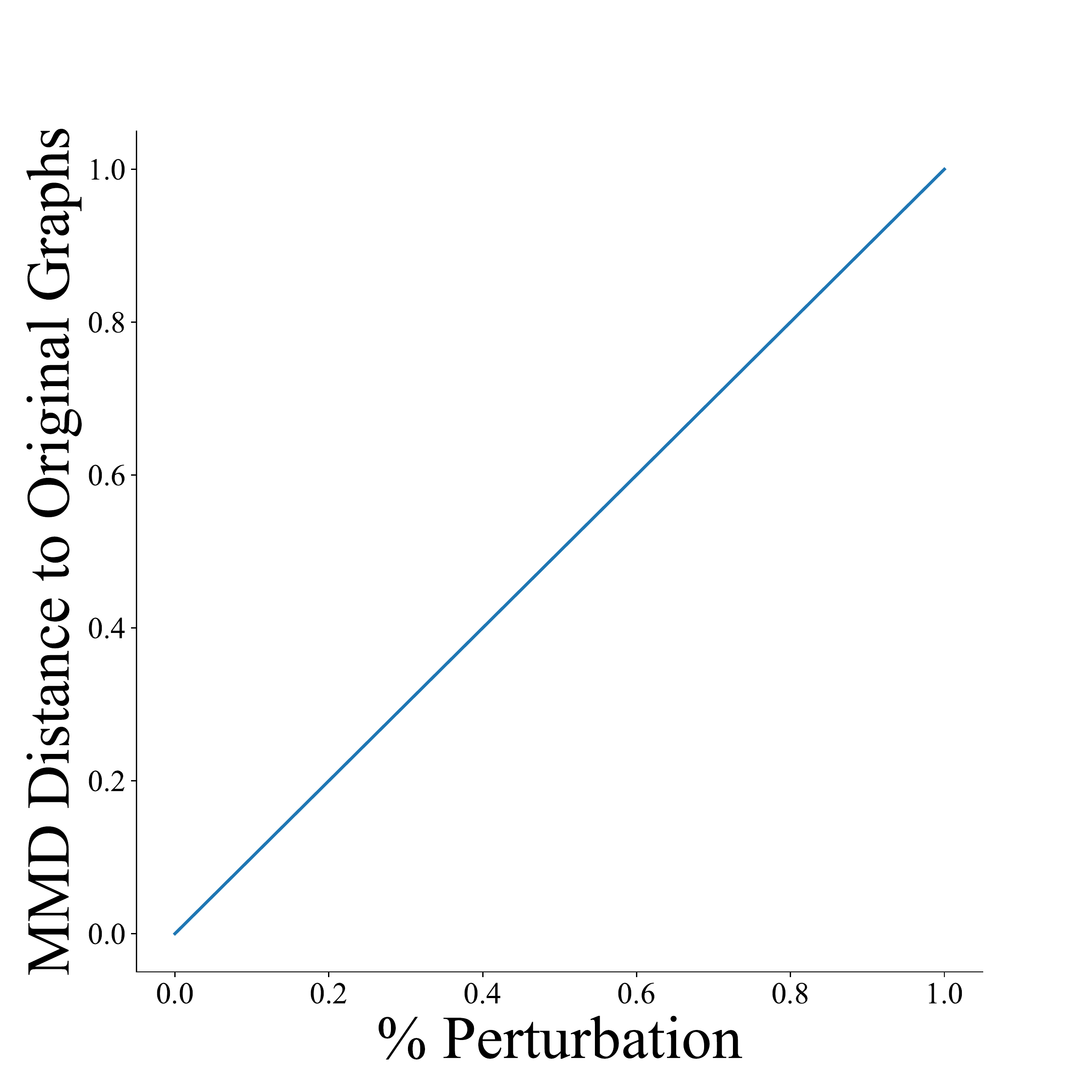}
        }%
        \caption{}
        \label{subfig:ideal_metric}
    \end{subfigure}%
        \hfill
    \begin{subfigure}[b]{0.33\textwidth}
        \resizebox{\textwidth}{!}{%
        \includegraphics[width=\textwidth]{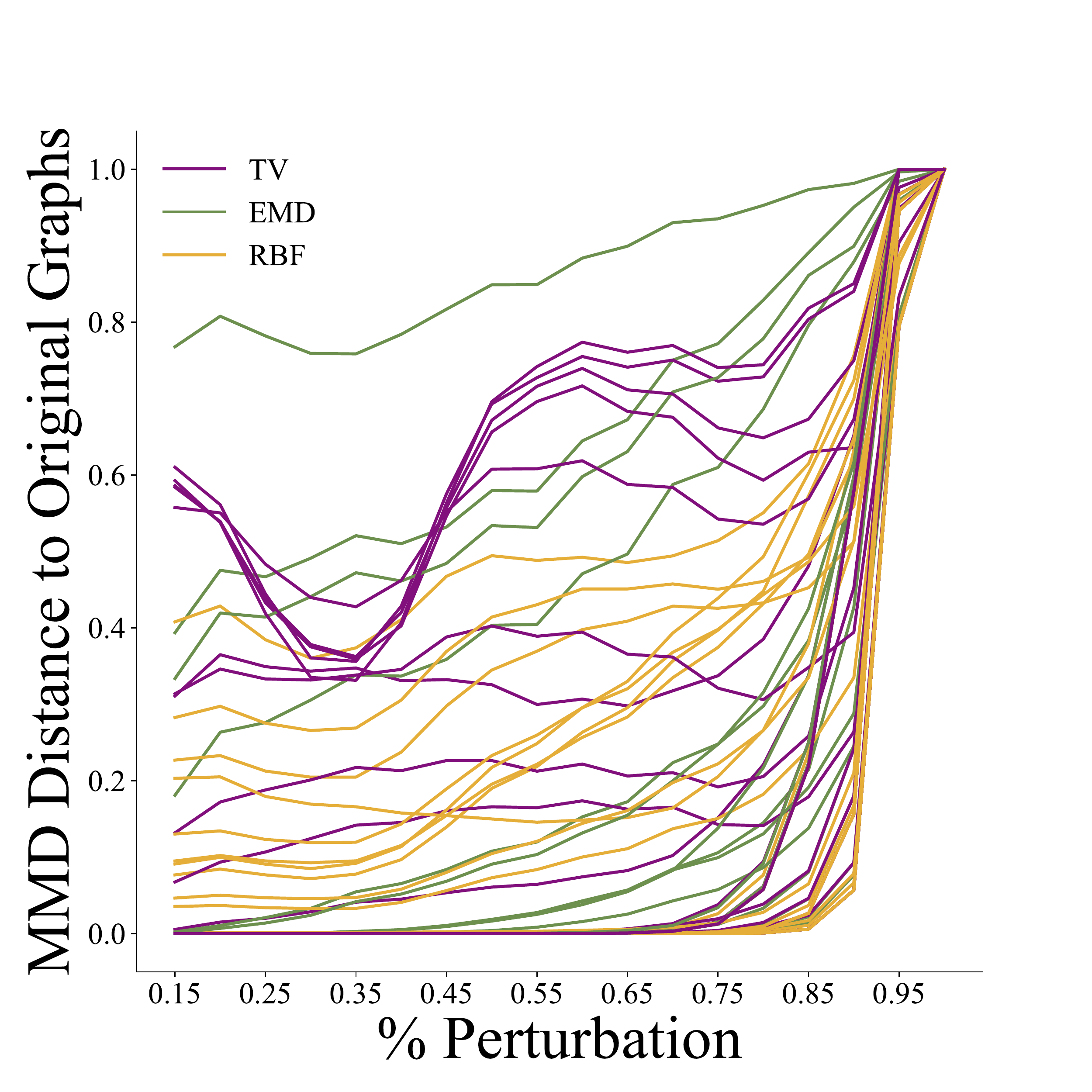}
        }%
        \caption{}
        \label{subfig:cc_mmd}
    \end{subfigure}%
    \hfill
    \begin{subfigure}[b]{0.33\textwidth}
        \resizebox{\textwidth}{!}{%
        \includegraphics[width=\textwidth]{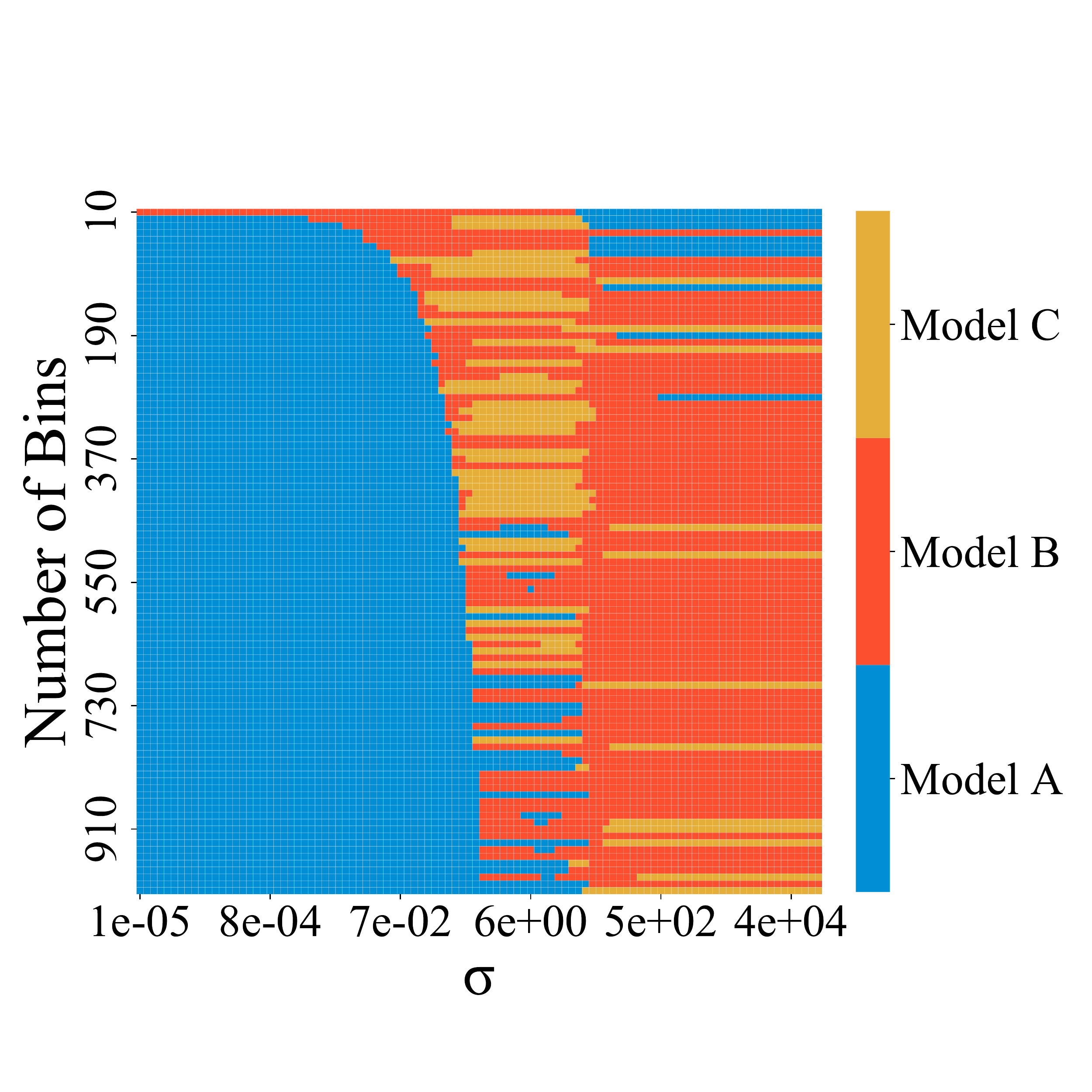}
        }%
        \caption{}
        \label{subfig:parameter_heatmap}
    \end{subfigure}%
    \caption{%
      \autoref{subfig:ideal_metric} shows the ideal behaviour of a graph generative model evaluator: as two distributions of graphs become increasingly dissimilar, e.g. via perturbations, the metric should grow proportionally. \autoref{subfig:cc_mmd} shows the behaviour of the current choices in reality; each line represents the normalized MMD for a given kernel and parameter combination. A cautious choice of kernel and parameters is needed in order to obtain a metric with behaviour similar to \autoref{subfig:ideal_metric}. Each square in \autoref{subfig:parameter_heatmap} shows which model performs best (out of A, B, and C) over a grid of hyperparameter combinations of $\sigma$ and number of bins in the histogram. Any model can rank first with an appropriate hyperparameter selection, showcasing the sensitivity of MMD to the hyperparameter choice.
    }
    \label{fig:MMD_pitfalls}
\end{figure}

\subsection{Nuances of using MMD for graph generative model evaluation}\label{subsec:mmd_issues}

While MMD may seem like a reasonable first choice as a metric for
comparing distributions of graphs, it is worth mentioning two
peculiarities of such a choice and how authors are currently applying it
in this context. First, MMD was originally developed as an
approach to perform two-sample testing on structured objects such as
graphs.  As such, its suitability was investigated in that context, and
\emph{not} in that of evaluating graph generative models. This warrants
an investigation of the implications of ``porting'' such a method from one
context to another.

The second peculiarity worth mentioning is that MMD was groundbreaking
for its ability to compare distributions of structured objects by means
of a kernel, thus bypassing the need to employ an intermediate vector
representation. Yet, in its current application, the graphs are first
being vectorised, and then MMD is used to compare the vectors.  MMD is
not technically required in this case: any statistical method
for comparing vector distributions could be employed, such as Optimal
Transport~(OT).
While the use of different evaluators besides MMD for assessing graph generative models is an interesting area for future research, we focus specifically on MMD, since this is what is in use today, and now highlight two practical issues that can arise from using MMD in this context.
\paragraph{MMD's ability to capture differences between distributions is kernel- and parameter-dependent.}
As two distributions become sufficiently dissimilar, we find that their
distance should monotonically increase as a function of their
dissimilarity.
While one can construct specific scenarios in which two distributions
become farther apart but the distance does not monotonically
increase~(e.g., removing edges from a triangle-free graph, and
using the clustering coefficient as the descriptor function),
in such cases, the specific choice of descriptor function~$f$ is crucial to
ensure that~$f$ can capture differences in distribution.  
However, it is not guaranteed that MMD will monotonically increase as two
distributions become increasingly dissimilar. \autoref{subfig:cc_mmd}
depicts this behaviour when focusing on a single descriptor
function, the clustering coefficient, with each line representing
a unique kernel and parameter choice~(see Appendix~\ref{sec:Full
perturbation results} for
the full results for other
datasets and descriptor functions). %
We subject an original set of graphs to perturbations of increasing
magnitude and then measured the MMD distance to the original distribution.
Despite both distributions becoming progressively dissimilar by
experimental design, a large number of kernel/parameter configurations
\emph{fail} to capture this, showing that MMD is highly sensitive to this choice.
In many instances, the distance remains nearly constant, despite the
increased level of perturbation, until the magnitude of the perturbation
reaches an extraordinarily high level.
In some cases, we observe that the distance even \emph{decreases} as the degree
of perturbation \emph{increases}, suggesting that the original data set
and its perturbed variant are more similar.
We also find that the MMD values as a function of the degree
of perturbation are highly sensitive to the kernel and
parameter selection, as evidenced by the wide range of different curve shapes
observed in \autoref{subfig:cc_mmd}.
%
\paragraph{MMD has no inherent scale.}
Another challenge with MMD is that since current practice works with the raw MMD distance, as opposed to p-values, as originally proposed in \citet{Gretton12}, there is no inherent scale of the MMD values. It is therefore difficult to assess whether the smaller MMD distance of
one model is \emph{substantially} improved when compared to the MMD
distance of another model.
For instance, suppose that the MMD distance of one model is
\num{5.2e-8}. Is this a meaningfully better model than one whose MMD
distance is \num{4.6e-7}? This is further compounded by the choice of
kernel and parameter, which also affects the scale of MMD. Since common heuristics can lead to a suboptimal choice of $\sigma$ for MMD~\citep{Sutherland17},
authors may obtain an arbitrarily low value of MMD,
as seen in Figure~\ref{subfig:emd_clustering_rank} and
\ref{subfig:gaussian_clustering_rank}.
As MMD results are typically reported just in a table, the lack of
a scale hinders the reader's ability to effectively assess a model's
performance.
%

\begin{figure}[tbp]
    \centering
    \begin{subfigure}[b]{0.25\textwidth}
        \resizebox{\textwidth}{!}{%
        \includegraphics[width=\textwidth]{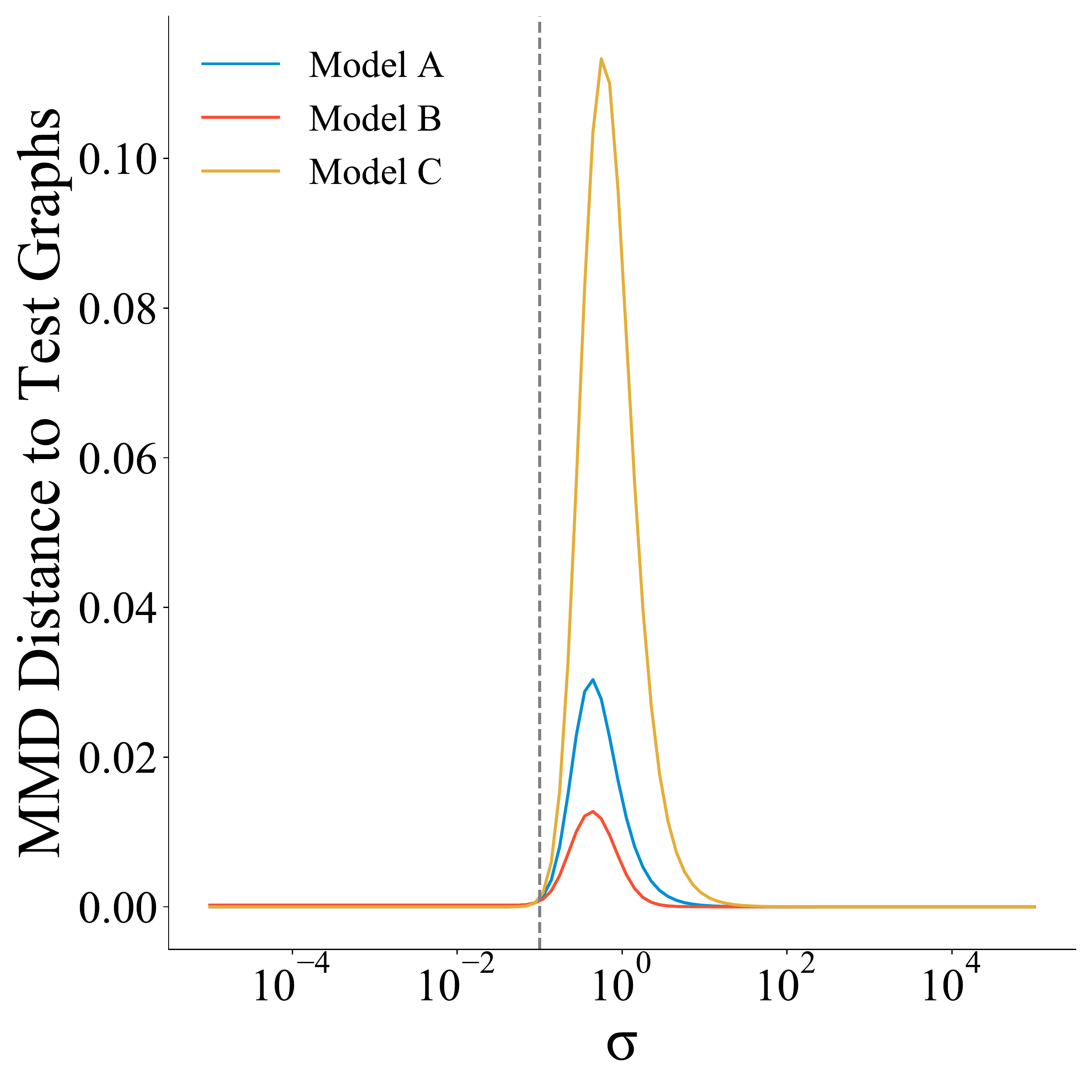}
        }%
    \end{subfigure}%
    \hfill
    \begin{subfigure}[b]{0.25\textwidth}
        
        \resizebox{\textwidth}{!}{%
        \includegraphics[width=\textwidth]{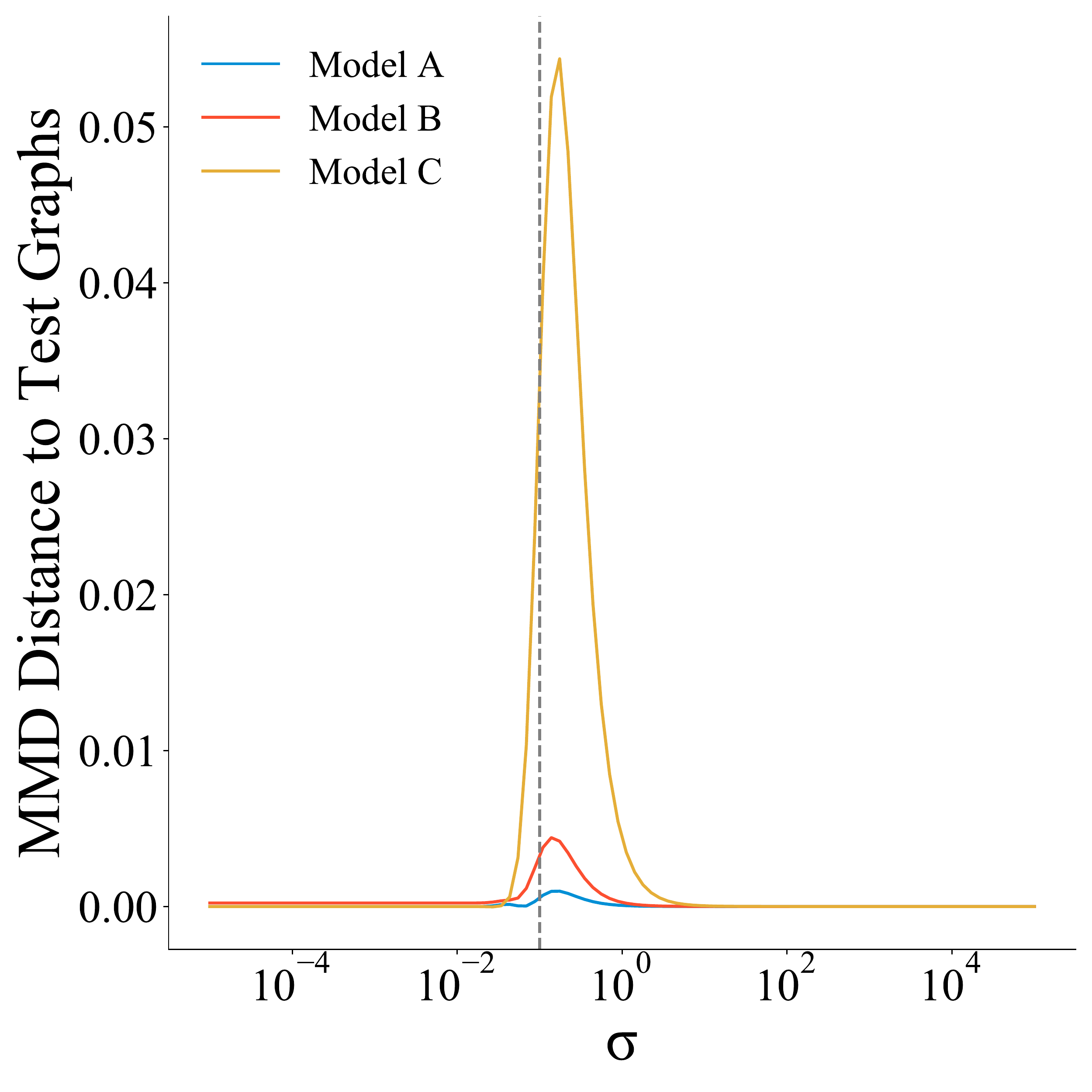}
        }%
    \end{subfigure}%
    \hfill%
    \centering
    \begin{subfigure}[b]{0.25\textwidth}
        \resizebox{\textwidth}{!}{%
        \includegraphics[width=\textwidth]{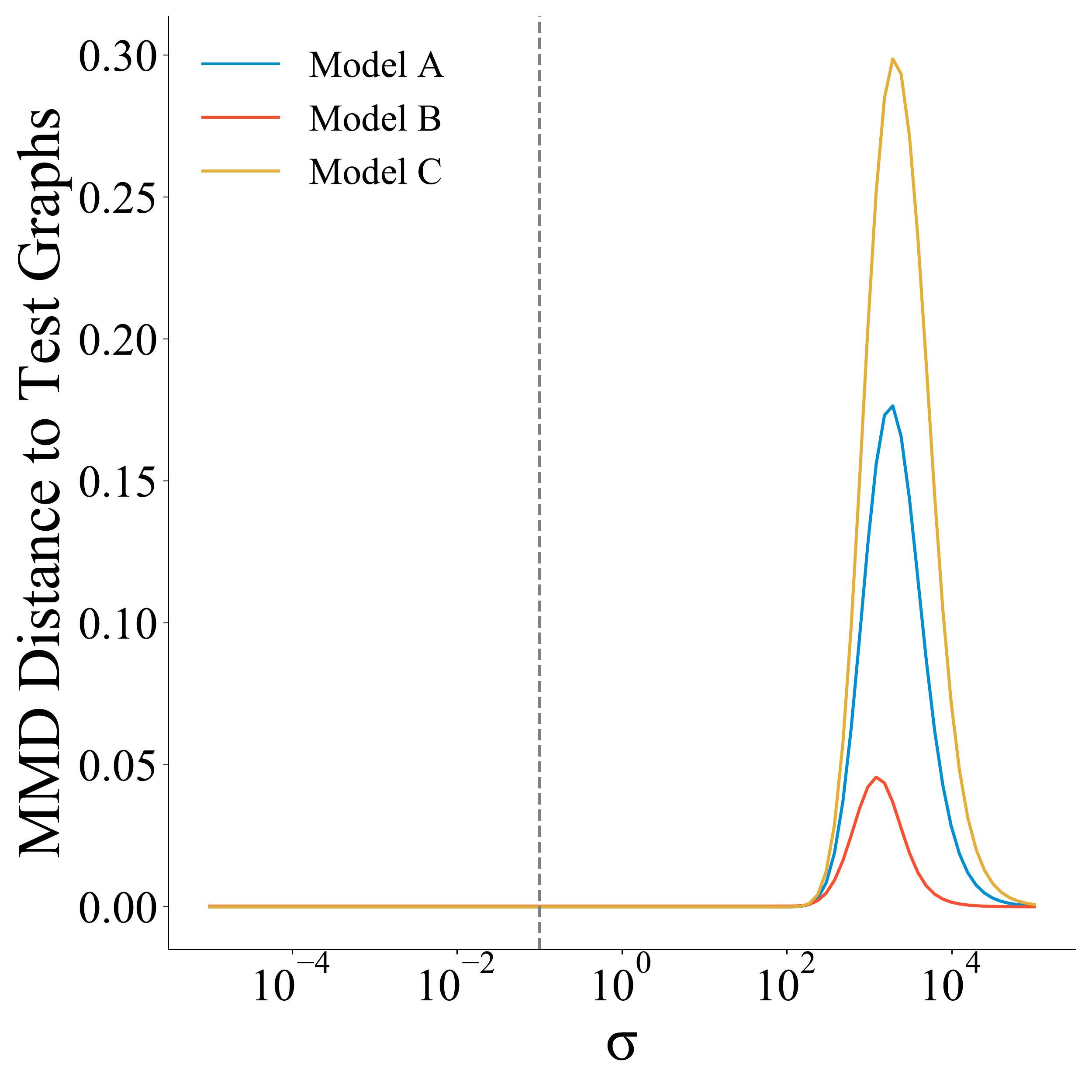}
        }%
    \end{subfigure}%
    \hfill
    \begin{subfigure}[b]{0.25\textwidth}
        \resizebox{\textwidth}{!}{%
        \includegraphics[width=\textwidth]{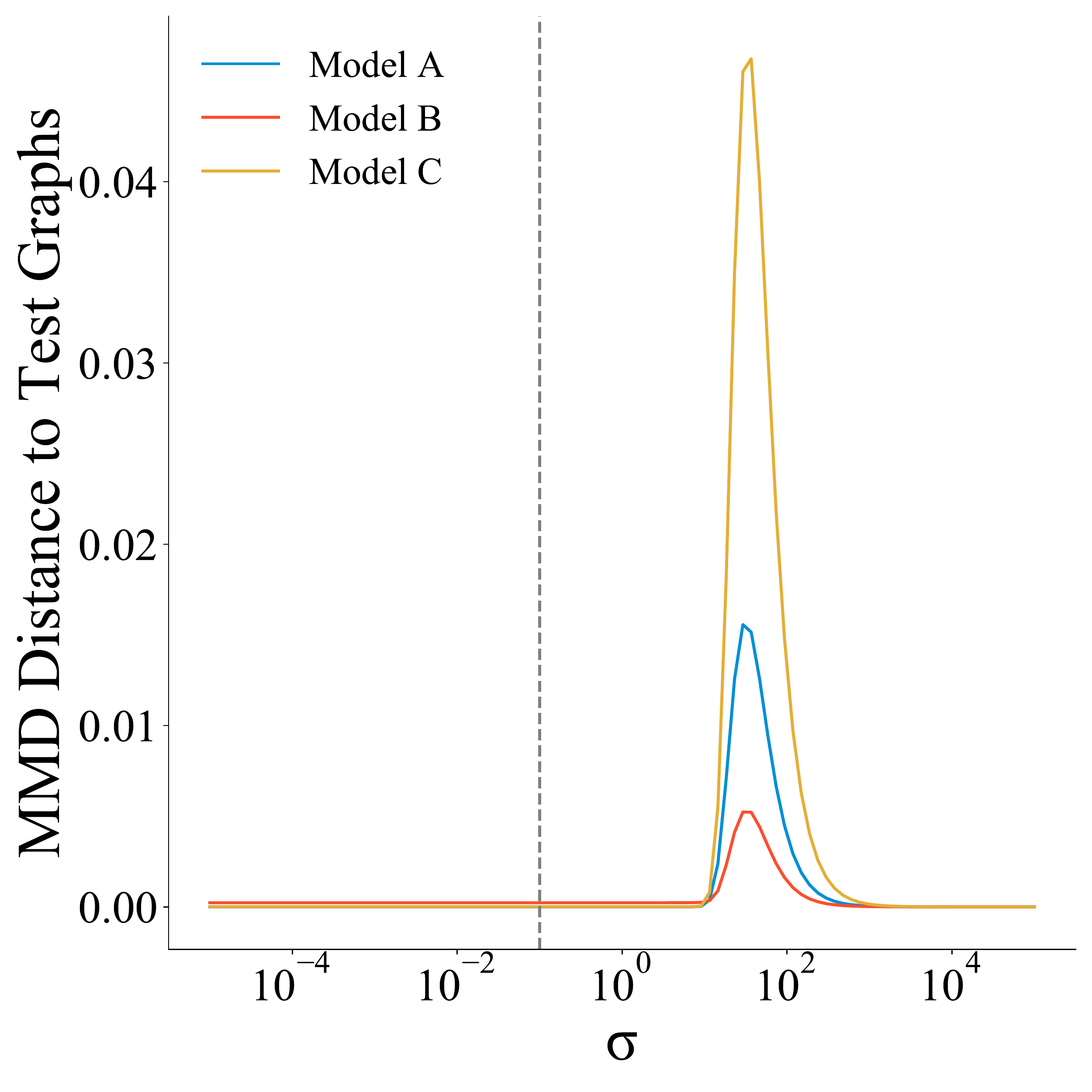}
        }%
    \end{subfigure}%
    \\
    \hspace{11pt}
    \begin{subfigure}[b]{0.215\textwidth}
        \resizebox{\textwidth}{!}{%
        \includegraphics[width=\textwidth]{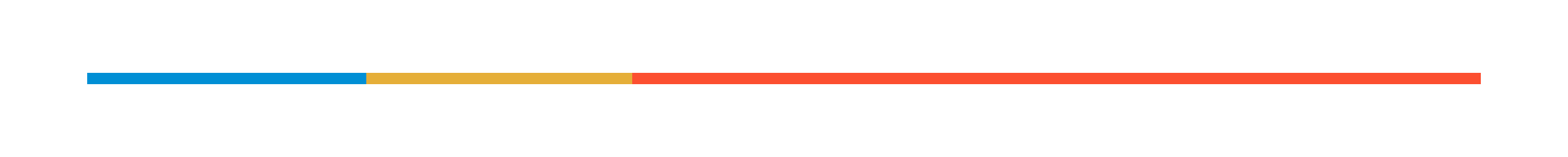}
        }%
        \caption{EMD, degree}
        \label{subfig:emd_degree_rank}
    \end{subfigure}%
    \hfill
        \hspace{11pt}
    \begin{subfigure}[b]{0.215\textwidth}

        \resizebox{\textwidth}{!}{%
        \includegraphics[width=\textwidth]{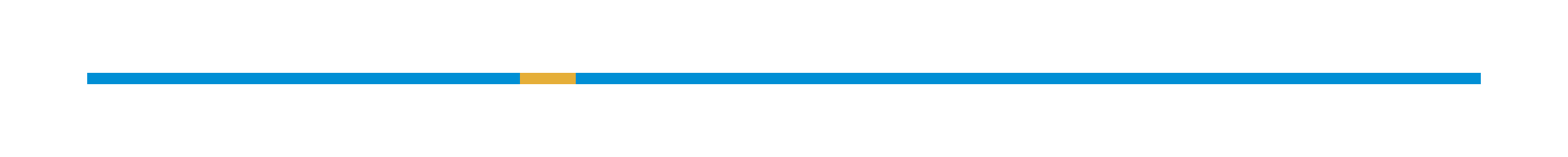}
        }%
        \caption{RBF, degree}
        \label{subfig:gaussian_degree_rank}
    \end{subfigure}%
    \hfill
        \hspace{11pt}
    \begin{subfigure}[b]{0.215\textwidth}
        \resizebox{\textwidth}{!}{%
        \includegraphics[width=\textwidth]{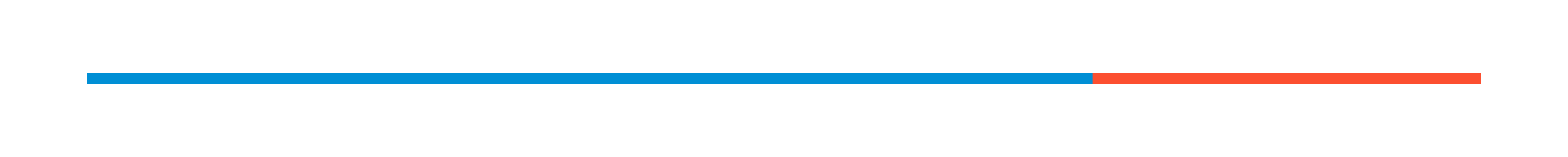}
        }%
        \caption{EMD, clust. coef.}
        \label{subfig:emd_clustering_rank}
    \end{subfigure}%
    \hfill
    \hspace{11pt}
    \begin{subfigure}[b]{0.215\textwidth}
        \resizebox{\textwidth}{!}{%
        \includegraphics[width=\textwidth]{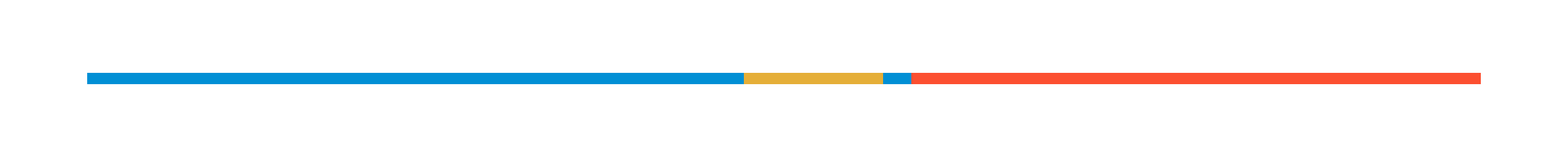}
        }%
        \caption{RBF, clust. coef.}
        \label{subfig:gaussian_clustering_rank}
    \end{subfigure}%
    \caption{%
      This shows the MMD distance to the test set of graphs for three recent graph generative models~(whose
      names we intentionally omitted) on the Community Graphs dataset for different descriptor functions and kernels. MMD requires the choice of a kernel and kernel parameters. Each subfigure shows MMD (lower is better) along a range of values of $\sigma$ (reported on a $\log$ scale), with the bar underneath indicating which model ranks first for the given value of $\sigma$. The grey line indicates the $\sigma$ chosen by the authors.  Subfigures~\ref{subfig:emd_degree_rank} and \ref{subfig:gaussian_degree_rank} show how simply switching from the EMD to the RBF kernel (holding $\sigma$ constant) can change which model performs best; Subfigures~\ref{subfig:emd_clustering_rank} and \ref{subfig:gaussian_clustering_rank} show how the choice of $\sigma$ by the authors misses the area of maximum discrimination of MMD.
    }
    \label{fig:kernel_sigma_pitfalls}
\end{figure}

\subsection{Consequences of the choice of kernel}\label{subsec:kernel_issues}

MMD requires the choice of a kernel function, yet there is no agreed-upon method
to select one. We find that each of the three models considered in this
paper choose a different kernel for its evaluation, as evidenced in
Table~\ref{tab:kernel_choice}. There are several potential issues with
the current practice, which we will subsequently discuss, namely
\begin{inparaenum}[(i)]
  \item\label{item:Efficiency} the selected kernel might be computationally inefficient to
    compute, as is the case for the previously-used Earth Mover's
    Distance~(EMD) kernel function,
  \item\label{item:Validity} the use of positive definite kernel functions is a limitation,
    with previous work employing functions that are
    \emph{not} valid kernels, and
   \item\label{item:Kernel ranking} finally, the kernel choice may result in arbitrary rankings, and there 
     is currently little attention paid to this choice.
\end{inparaenum}
\paragraph{Computational cost of kernel computation.} Issue~(\ref{item:Efficiency}) might prevent an evaluation metric to be
used in practice, thus stymieing graph generative model development.
While the choice of a kernel using the first Wasserstein distance is
a valid choice, it is extremely slow to compute, as noted by
\cite{Liao19}. From this perspective, it violates the third quality of our desiderata: efficiency.
\paragraph{Kernels need to be p.s.d.}
To reduce the aforementioned computational costs, previous
work~\citep{Liao19} used a kernel based on the total variation distance
between histograms, which has since been used in subsequent publications
as one of the ways to evaluate different models.
As we show in Appendix~\ref{sec:Total variation distance}, this approach
leads to an \emph{indefinite kernel}~(i.e.\ the kernel is neither
positive definite nor negative definite), whose behaviour in the context
of MMD is not well-defined. MMD necessitates the use of p.s.d. kernels,
and care must be taken when it comes to interpreting the respective
results. 
\paragraph{Arbitrary ranking based on kernel choice.}
Issue~(\ref{item:Kernel ranking}) relates to the fact that changing the
kernel can lead to different results of model evaluation, which is
problematic since each paper we considered used a different kernel. For
instance, with the degree distribution as the graph descriptor function,
simply changing the choice of kernel from EMD to RBF~(while holding the
parameters constant) leads to a different ranking of the models, which
can be seen in
\autoref{subfig:emd_degree_rank}--\autoref{subfig:gaussian_degree_rank},
where the best performing model changes merely by changing the kernel
choice~(!). This type of behaviour is highly undesired and problematic,
as it implies that model performance~(in terms of the evaluation) can be
improved by choosing a different kernel.

\subsection{Effect of the choice of hyperparameters}\label{subsec:parameter_issues}

While the choice of which kernel to use in the MMD calculation is itself
an \emph{a priori} design choice without clear justification, it is
further exacerbated by the fact that many kernels require picking
parameters, without any clear process by which to
choose them. To the best of our knowledge, this selection of parameters
is glossed over in publications at present.\footnote{%
  We remark that literature outside the graph generative modelling
  domain describes such choices, for example in the context of
  two-sample tests~\citep{Gretton12a} or general model
  criticism~\citep{Sutherland17}. We will subsequently discuss to what
  extent the suggested parameter selection strategies may be
  transferred.
}
For instance, Table~\ref{tab:kernel_choice} shows how authors are setting the
value of $\sigma$ differently for different descriptor functions, yet
there is no discussion nor established best practice of \emph{how} such
parameters were or should be set. Hyperparameter selection is known to
be a crucial issue in machine learning that mandates clear selection
algorithms in order to avoid biasing the results.
If the choice of parameters---similar to the choice of kernel---had no
bearing on the outcome, this would not be problematic, but our empirical
experiments prove that drastic differences in model evaluation
performance can occur.
\paragraph{Arbitrary ranking based on parameter choice.}
The small colour bars underneath each plot of
Figure~\ref{fig:kernel_sigma_pitfalls} show the model that achieves the
lowest MMD for a given value of $\sigma$. Changes in colour highlight
a sensitivity to specific parameter values. Even though the plots showing MMD seem to have a general trend of which model is best in the peak of the curves, in the other regions of $\sigma$, the ranking switches, with the effect that a different model would be declared ``best.'' This is particularly the case in Subfigures~\ref{subfig:emd_degree_rank}, \ref{subfig:emd_clustering_rank}, and \ref{subfig:gaussian_clustering_rank}, where Model B appears to be the best, yet for much of $\sigma$, a different model ranks first.
This sensitivity is further exacerbated by the fact that some
descriptor functions require a parameter choice as well, such as the bin
size, $n_\mathrm{bin}$, for a histogram. 
Figure~\ref{subfig:parameter_heatmap}
shows how the best-ranking model on the Barab\'{a}si-Albert Graphs is entirely dependent upon the
choice of parameters ($n_\mathrm{bin}$, $\sigma$).
The colour in each grid cell corresponds to the best-ranking model for a given $\sigma$ and $n_\mathrm{bin}$; we
find that this is wildly unstable across both $n_\mathrm{bin}$ and $\sigma$.
The consequence of this is alarming: \emph{any} model could rank first if the
right parameters are chosen.
\paragraph{Choice of $\sigma$ by authors does \emph{not} align with maximum
discrimination in MMD.}
Figures~\ref{subfig:emd_clustering_rank} and
\ref{subfig:gaussian_clustering_rank} show MMD for different values of
$\sigma$ with the clustering coefficient descriptor function in the
Community Graphs dataset for Models A, B and C. The value of $\sigma$
as selected by the authors~(without justification or discussion) is indicated
by the grey line; however, this choice corresponds to a regions of low
activity in the MMD curve, suggesting a poor parameter choice. While Model~B seems to be the clear winner, the choice of $\sigma$ by the authors resulted in Model~A having
the best performance. Furthermore, it is not clear that choosing the
same $\sigma$ across different kernels, as is currently done, makes
sense; whereas $\sigma = 10^3$ would be sensible for EMD in
Figure~\ref{subfig:emd_clustering_rank}, for the Gaussian kernel in
Figure~\ref{subfig:gaussian_clustering_rank}, such a choice too far
beyond the discriminative peak.

\section{How to use MMD for graph generative model evaluation}\label{sec:Practical_recommendations}

Having understood the potential pitfalls of using MMD, we now turn to
suggestions on how to better leverage MMD for graph generative model
evaluation. 
\paragraph{Provide a sense of scale.}
As mentioned in Section~\ref{subsec:mmd_issues}, MMD does not have an
inherent scale, making it difficult to assess what is `good.'
\added{To endow their results with some meaning, practitioners should calculate MMD between the test and training graphs, and then include this in the results table/figures alongside the other MMD results. This will provide a meaningful bound on what two  `indistinguishable' sets of graphs look like in a given dataset~(see Figures~\ref{fig:Full perturbation results: MMD, Adding Edges}--\ref{fig:Full perturbation results: MMD, Add Connected Nodes} in Appendix~\ref{sec:Full perturbation results}}).
\paragraph{Choose valid and efficient kernel candidates.}
We recommend to avoid the EMD-based kernel due to the computational burden \added{(see Appendix~\ref{sec:kernel_runtime})}, and the total variation kernel for its non-p.s.d nature.
%
Instead, we suggest using either an RBF
kernel, \added{since it is a universal kernel,} or a Laplacian kernel, or a linear
kernel, i.e.\ the canonical inner product on $\reals^d$, since it is
parameter-free. As all of these kernels are p.s.d.,\footnote{%
  In the case of the Laplacian kernel, the TV
  distance in lieu of the Euclidean distance leads to a valid
  kernel.
}
and are fast to compute, they
satisfy the efficiency desiderata criteria, and thus only require 
analysis of their expressivity and robustness.
\paragraph{Utilize meaningful descriptor functions.}
Different descriptor functions measure different aspects of the graph
and are often domain-specific. As a general recommendation, we propose
using previously-described~\citep{You18, Liao19}  graph-level descriptor
functions, namely
\begin{inparaenum}[(i)]
  \item the degree distribution,
  \item the clustering coefficient, and
  \item the Laplacian spectrum histograms,
\end{inparaenum}
and recommend that the practitioner make domain-specific adjustments based on what is
appropriate.
%
%
%

\subsection{Selecting an appropriate kernel and hyperparameters}

The kernel choice and descriptor functions require hyperparameter
selection. We recommend assessing the performance of MMD in
a controlled setting to elucidate some of the properties specific to the
dataset and the descriptor functions of interest to a given application. In
doing so, it becomes possible to choose a kernel and parameter
combination that will yield informative results.
Notice that in contrast to images, where visualisation provides
a meaningful evaluation of whether two images are similar or not, graphs
cannot be assessed in this manner. It is thus necessary to have
a principled approach where the degree of difference between two
distributions can be controlled. We subject a set of graphs to perturbations~(edge insertions,
removals, etc.) of increasing
magnitude, thus enabling us to assess the expected degree of difference
to the original graphs.

We ideally want an evaluation metric to effectively reflect the degree
of perturbations.  Hence, with an increasing degree of
perturbation of graphs,
the distance to the original distribution~$\graphs^\ast$ of unperturbed
graphs should increase.
We can therefore assess both the expressivity of the evaluation metric, i.e., its
ability to distinguish two distributions when they are different, and
its robustness~(or stability) based on how rapidly such a metric changes
when subject to small perturbations. Succinctly, we would like to see
a clear correlation of the metric with the degree of perturbation,
thus indicating both robustness and expressivity.
%
Perturbation experiments are particularly appealing because they do not require access to other models but rather only an initial distribution~$\graphs^\ast$.
This procedure therefore does not leak any information from
models and is unbiased. Moreover, researchers have more
control over the ground truth in this scenario, as they can adjust
the desired degree of dissimilarity.
While there are many perturbations that we will consider (adding edges,
removing edges, rewiring edges, and adding connected nodes), we focus
primarily on progressively adding or removing edges. This is the graph
analogue to adding ``salt-and-pepper'' noise to an image, and in its most extreme
form (100\% perturbation) corresponds to a fully-connected graph and
fully-disconnected graph, respectively. 
\paragraph{\added{Creating dissimilarity via perturbations}.}
\added{For each perturbation type, i.e., }%
\begin{inparaenum}[(i)]
  \item \added{random edge insertions,} 
  \item \added{random edge deletions,} 
  \item \added{random rewiring operations~(`swapping' edges)},
  \item \added{random node additions,}
\end{inparaenum}
\added{we progressively perturb the set of graphs, using the relevant perturbation parameters, in order to obtain multiple sets of graphs that are increasingly dissimilar from the original set of graphs. Each perturbation is parametrised by at least one parameter. When removing edges, for instance,
the parameter is the probability of removing an edge in the graph.  Thus for a graph
with $100$ edges and $p_\text{remove}=0.1$ we would expect on average $90$ edges to
remain in the graph.  Similar parametrisations apply for the other
perturbations, i.e.\ the probability of adding an edge for edge insertions, the probability of rewiring an edge for edge rewiring, and the number of nodes to add to a graph, as well as the probability of an edge between the new node and the other nodes in the graph for adding connected nodes. We provide a formal description of each process in Appendix~\ref{sec:Experimental Setup} and \ref{sec:Details graph perturbations}.}
\paragraph{Correlation analysis to choose a kernel and hyperparameters.}
For each of the aforementioned perturbation types, we compared the graph
distribution of the perturbed graphs with the original graphs using the
MMD.  We repeated this for different scenarios,
comprising different kernels, different descriptor functions, and
where applicable, parameters. For a speed up trick to efficiently calculate MMD over a range of $\sigma$, please see Appendix~\ref{sec:Speed up}.
Since these experiments resulted in hundreds of
configurations, due to the choice of kernel, descriptor
function, and parameter choices, we relegated most of the visualisations
to the Appendix~(see \autoref{sec:Full perturbation results}).
To compare the different configurations effectively, we
needed a way to condense the multitude of results into a more
interpretable and comparable visualisation. We therefore calculated
Pearson's correlation coefficient between the degree of perturbation and
the resulting MMD distance, 
%
obtaining  two heatmaps. The first one shows the \emph{best
parameter choice}, the second one shows the \emph{worst parameter
choice}, both measured in terms of Pearson's correlation coefficient~(\autoref{fig:Perturbation Heatmap}).
In the absence of an agreed-upon procedure to choose such
parameters, the heatmaps effectively depict the extremes of what will
happen if one is particularly ``lucky'' or ``unlucky'' in the choice of
parameters. A robust combination of descriptor function and comparison
function is characterised by both heatmaps exhibiting high correlation
with the degree of perturbation.  In the bottom row, the MMD distance in many cases no longer shows
\emph{any correlation} with the degree of perturbation, and in the case of the
clustering coefficient~(CC), is even negatively correlated with the
degree of perturbation. Such behaviour is undesired, showcasing
a potential pitfall for authors if they inadvertently fail to 
pick a ``good'' parameter or kernel combination. \added{While we chose the Pearson correlation coefficient for its simplicity and interpretability, other measures of dependence could be used instead if they are domain-appropriate~(see Appendix~\ref{sec:Alternative correlation measures}).}
To choose a kernel and parameter combination, we suggest picking the
parameters with the highest correlation for the perturbation that is
most meaningful in the given domain. Lacking this, the practitioner could choose the combination that has the
highest \emph{average} correlation across perturbations. 

\begin{figure}[tbp]
    \centering
    
    \begin{subfigure}[b]{0.24\textwidth}
        \includegraphics[height=4cm]{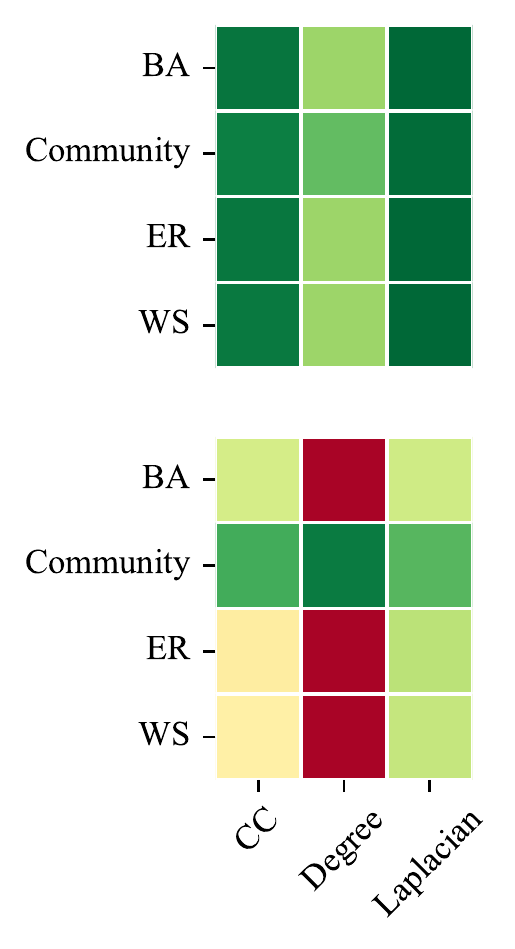}
        \caption{AddEdges}
    \end{subfigure}%
    \begin{subfigure}[b]{0.24\textwidth}
        \includegraphics[height=4cm]{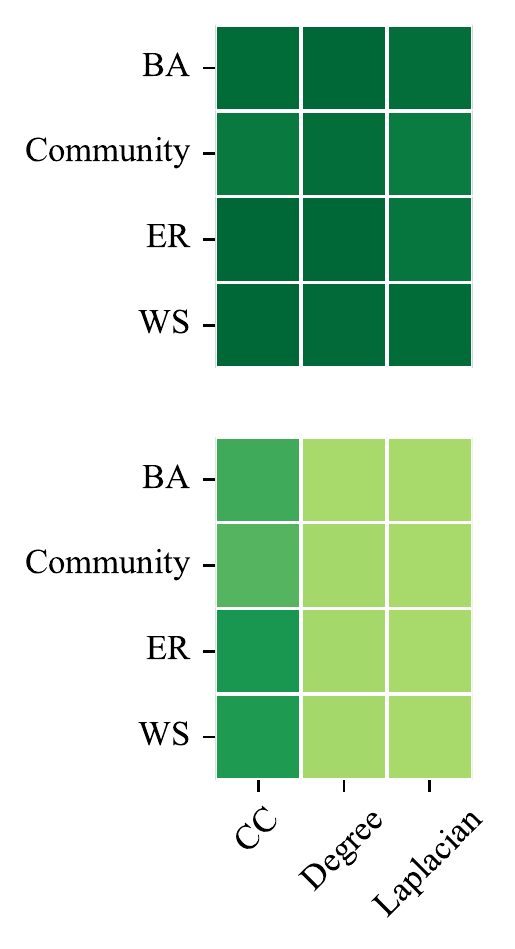}
        
        \caption{RemoveEdges}
    \end{subfigure}%
    \begin{subfigure}[b]{0.24\textwidth}
        \includegraphics[height=4cm]{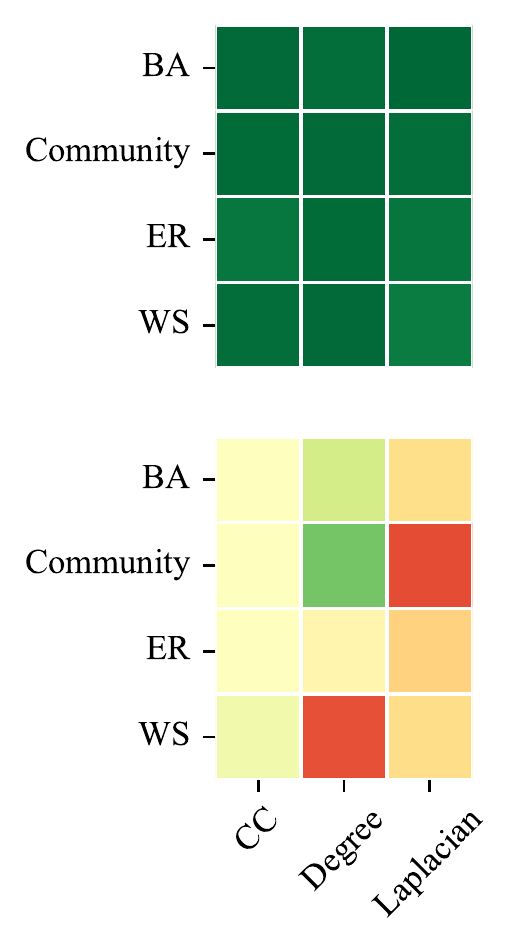}
        \caption{RewireEdges}
    \end{subfigure}
    \begin{subfigure}[b]{0.24\textwidth}
        \includegraphics[height=4cm]{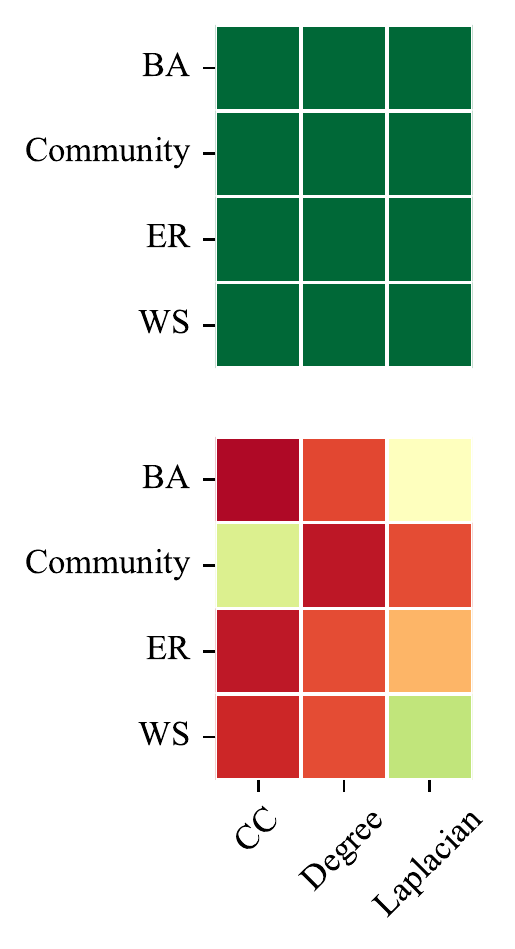}
        \caption{AddConnectedNodes}
    \end{subfigure}
    \begin{subfigure}[b]{0.02\textwidth}
        \includegraphics[height=4cm]{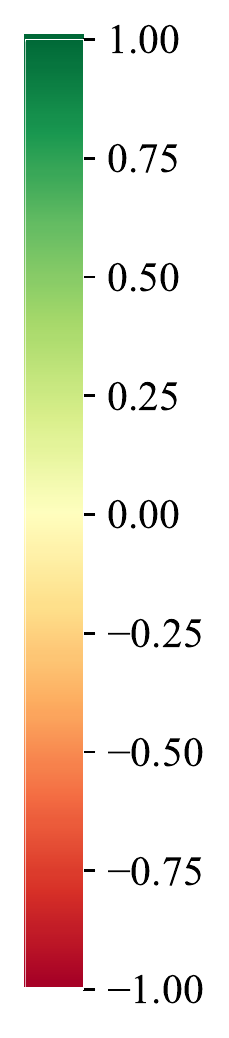}
        \caption*{}
    \end{subfigure}
    \caption{%
      The correlation of MMD with
      the degree of perturbation in the graph, assessed for different
      descriptor functions and datasets (BA: Barab\'{a}si-Albert, ER: Erd\"{o}s-R\'{e}nyi, WS: Watts-Strogatz). For an ideal metric, the
      distance would increase with the degree of perturbation,
      resulting in values $\approx 1$. The upper row shows the
      \emph{best} kernel-parameter combination; the bottom row shows the \emph{worst}. A proper kernel and parameter selection leads to strong correlation to the perturbation, but a bad choice can lead to inverse correlation, highlighting the importance of a good kernel/parameter combination. 
    }
    \label{fig:Perturbation Heatmap}
\end{figure}

\section{Conclusion}\label{sec:Discussion}
%
We provided a thorough analysis of how graph generative models are being
currently assessed by means of MMD. While MMD itself is powerful and
expressive, its use has
certain idiosyncratic issues that need to be avoided in order to obtain
fair and reproducible comparisons.
We highlighted some of these issues, most critical of which are that the
choice of kernel and parameters can result in different rankings of different
models, and that MMD may not monotonically increase as two graph distributions
become increasingly dissimilar.
As a mitigation strategy, we propose running a perturbation experiment
as described in this paper to select a kernel and parameter combination that is highly
correlated with the degree of perturbation. This way, the choice of
parameters does not depend on the candidate models but only on the
initial distribution of graphs.
\paragraph{Future work.}
This work gives an overview of the current situation,
illuminates some issues with the status quo, and provides practical
solutions. We hope that this will serve as a starting point for the
community to further develop methods to assess graph generative models, and it is encouraging that some efforts to do so are already underway~\citep{thompson2022on}.
Future work could investigate the use of efficient
graph kernels in combination with MMD, as described in a recent
review~\citep{Borgwardt20}. This would reduce the comparison pipeline in
that graph kernels can be directly used with MMD, making graph
descriptor functions unnecessary. Another approach could be to
investigate \added{alternative or new descriptor functions, such as the geodesic distance, or alternative evaluation methods, such as the multivariate Kolmogorov--Smirnov test \citep{Justel1997}, or even develop totally novel evaluation strategies},
preferably those that go beyond the currently-employed vectorial
representations of graphs. 
%
\section*{Reproducibility Statement}

We have provided the code for our experiments in order to make our work fully reproducible. Details can be found in Appendix~\ref{sec:Experimental Setup}-\ref{sec:Implementation details}, which includes a link to our GitHub repository. During the review process our code is available as a Supplementary Material, in order to preserve anonymity.


\subsubsection*{Acknowledgments}
This work was supported in part by the Alfried Krupp Prize for Young University Teachers of the Alfried Krupp von Bohlen und Halbach-Stiftung (K.B.). 

\bibliography{iclr2022_conference}

\begin{thebibliography}{28}
\providecommand{\natexlab}[1]{#1}
\providecommand{\url}[1]{\texttt{#1}}
\expandafter\ifx\csname urlstyle\endcsname\relax
  \providecommand{\doi}[1]{doi: #1}\else
  \providecommand{\doi}{doi: \begingroup \urlstyle{rm}\Url}\fi

\bibitem[Borgwardt et~al.(2020)Borgwardt, Ghisu, Llinares-L\'{o}pez, O'Bray,
  and Rieck]{Borgwardt20}
Karsten Borgwardt, Elisabetta Ghisu, Felipe Llinares-L\'{o}pez, Leslie O'Bray,
  and Bastian Rieck.
\newblock Graph kernels: State-of-the-art and future challenges.
\newblock \emph{Foundations and Trends in Machine Learning}, 13\penalty0
  (5--6):\penalty0 531--712, 2020.

\bibitem[Borgwardt et~al.(2006)Borgwardt, Gretton, Rasch, Kriegel, Schölkopf,
  and Smola]{Borgwardt06}
Karsten~M. Borgwardt, Arthur Gretton, Malte~J. Rasch, Hans-Peter Kriegel,
  Bernhard Schölkopf, and Alex~J. Smola.
\newblock Integrating structured biological data by kernel maximum mean
  discrepancy.
\newblock \emph{Bioinformatics}, 22\penalty0 (14):\penalty0 e49--e57, 2006.

\bibitem[Bounliphone et~al.(2016)Bounliphone, Belilovsky, Blaschko, Antonoglou,
  and Gretton]{Bounliphone15}
Wacha Bounliphone, Eugene Belilovsky, Matthew~B. Blaschko, Ioannis Antonoglou,
  and Arthur Gretton.
\newblock A test of relative similarity for model selection in generative
  models.
\newblock In \emph{International Conference on Learning Representations}, 2016.

\bibitem[Bridson \& Haefliger(1999)Bridson and Haefliger]{Bridson99}
Martin~R. Bridson and Andr{\'e} Haefliger.
\newblock The model spaces {$M_\kappa^n$}.
\newblock In \emph{Metric Spaces of Non-Positive Curvature}, pp.\  15--31.
  Springer, Berlin, Heidelberg, 1999.
\newblock ISBN 978-3-662-12494-9.
\newblock \doi{10.1007/978-3-662-12494-9_2}.

\bibitem[Chen et~al.(2021)Chen, Han, Hu, Ruiz, and Liu]{Chen2021}
Xiaohui Chen, Xu~Han, Jiajing Hu, Francisco Ruiz, and Liping Liu.
\newblock Order matters: Probabilistic modeling of node sequence for graph
  generation.
\newblock In Marina Meila and Tong Zhang (eds.), \emph{Proceedings of the 38th
  International Conference on Machine Learning}, volume 139 of
  \emph{Proceedings of Machine Learning Research}, pp.\  1630--1639. PMLR,
  18--24 Jul 2021.

\bibitem[Chung(1997)]{Chung97}
Fan R.~K. Chung.
\newblock \emph{Spectral Graph Theory}, volume~92 of \emph{CBMS Regional
  Conference Series in Mathematics}.
\newblock American Mathematical Society, 1997.

\bibitem[Dai et~al.(2020)Dai, Nazi, Li, Dai, and Schuurmans]{Dai2020}
Hanjun Dai, Azade Nazi, Yujia Li, Bo~Dai, and Dale Schuurmans.
\newblock Scalable deep generative modeling for sparse graphs.
\newblock In Hal~Daumé III and Aarti Singh (eds.), \emph{Proceedings of the
  37th International Conference on Machine Learning}, volume 119 of
  \emph{Proceedings of Machine Learning Research}, pp.\  2302--2312. PMLR,
  13--18 Jul 2020.

\bibitem[Feragen et~al.(2015)Feragen, Lauze, and Hauberg]{Feragen15}
Aasa Feragen, Fran{\c c}ois Lauze, and S{\o}ren Hauberg.
\newblock Geodesic exponential kernels: When curvature and linearity conflict.
\newblock In \emph{IEEE Conference on Computer Vision and Pattern
  Recognition~(CVPR)}, pp.\  3032--3042, 2015.

\bibitem[Goyal et~al.(2020)Goyal, Jain, and Ranu]{Goyal2020}
Nikhil Goyal, Harsh~Vardhan Jain, and Sayan Ranu.
\newblock Graphgen: A scalable approach to domain-agnostic labeled graph
  generation.
\newblock In \emph{Proceedings of The Web Conference 2020}, WWW '20, pp.\
  1253–1263. Association for Computing Machinery, 2020.

\bibitem[Gretton et~al.(2007)Gretton, Borgwardt, Rasch, Sch\"{o}lkopf, and
  Smola]{Gretton07}
Arthur Gretton, Karsten Borgwardt, Malte Rasch, Bernhard Sch\"{o}lkopf, and
  Alex~J. Smola.
\newblock A kernel method for the two-sample-problem.
\newblock In B.~Sch\"{o}lkopf, J.~C. Platt, and T.~Hoffman (eds.),
  \emph{Advances in Neural Information Processing Systems~19}, pp.\  513--520.
  MIT Press, 2007.

\bibitem[Gretton et~al.(2012{\natexlab{a}})Gretton, Borgwardt, Rasch,
  Sch{{\"o}}lkopf, and Smola]{Gretton12}
Arthur Gretton, Karsten~M. Borgwardt, Malte~J. Rasch, Bernhard Sch{{\"o}}lkopf,
  and Alexander Smola.
\newblock A kernel two-sample test.
\newblock \emph{Journal of Machine Learning Research}, 13\penalty0
  (25):\penalty0 723--773, 2012{\natexlab{a}}.

\bibitem[Gretton et~al.(2012{\natexlab{b}})Gretton, Sejdinovic, Strathmann,
  Balakrishnan, Pontil, Fukumizu, and Sriperumbudur]{Gretton12a}
Arthur Gretton, Dino Sejdinovic, Heiko Strathmann, Sivaraman Balakrishnan,
  Massimiliano Pontil, Kenji Fukumizu, and Bharath~K. Sriperumbudur.
\newblock Optimal kernel choice for large-scale two-sample tests.
\newblock In F.~Pereira, C.~J.~C. Burges, L.~Bottou, and K.~Q. Weinberger
  (eds.), \emph{Advances in Neural Information Processing Systems}, volume~25.
  Curran Associates, Inc., 2012{\natexlab{b}}.

\bibitem[Gromov(1987)]{Gromov87}
Mikhail Gromov.
\newblock Hyperbolic groups.
\newblock In S.~M. Gersten (ed.), \emph{Essays in Group Theory}, pp.\  75--263.
  Springer, Heidelberg, Germany, 1987.

\bibitem[Heusel et~al.(2017)Heusel, Ramsauer, Unterthiner, Nessler, and
  Hochreiter]{Heusel17}
Martin Heusel, Hubert Ramsauer, Thomas Unterthiner, Bernhard Nessler, and Sepp
  Hochreiter.
\newblock {GANs} trained by a two time-scale update rule converge to a local
  {N}ash equilibrium.
\newblock In I.~Guyon, U.~V. Luxburg, S.~Bengio, H.~Wallach, R.~Fergus,
  S.~Vishwanathan, and R.~Garnett (eds.), \emph{Advances in Neural Information
  Processing Systems~30}, pp.\  6626--6637. Curran Associates, Inc., 2017.

\bibitem[Justel et~al.(1997)Justel, Peña, and Zamar]{Justel1997}
Ana Justel, Daniel Peña, and Rubén Zamar.
\newblock A multivariate {K}olmogorov-{S}mirnov test of goodness of fit.
\newblock \emph{Statistics \& Probability Letters}, 35\penalty0 (3):\penalty0
  251--259, 1997.
\newblock ISSN 0167-7152.

\bibitem[Liao et~al.(2019)Liao, Li, Song, Wang, Hamilton, Duvenaud, Urtasun,
  and Zemel]{Liao19}
Renjie Liao, Yujia Li, Yang Song, Shenlong Wang, Will Hamilton, David~K
  Duvenaud, Raquel Urtasun, and Richard Zemel.
\newblock Efficient graph generation with graph recurrent attention networks.
\newblock In H.~Wallach, H.~Larochelle, A.~Beygelzimer, F.~d\textquotesingle
  Alch\'{e}-Buc, E.~Fox, and R.~Garnett (eds.), \emph{Advances in Neural
  Information Processing Systems}, volume~32. Curran Associates, Inc., 2019.

\bibitem[Lloyd \& Ghahramani(2015)Lloyd and Ghahramani]{Lloyd15}
James~R Lloyd and Zoubin Ghahramani.
\newblock Statistical model criticism using kernel two sample tests.
\newblock In C.~Cortes, N.~Lawrence, D.~Lee, M.~Sugiyama, and R.~Garnett
  (eds.), \emph{Advances in Neural Information Processing Systems}, volume~28.
  Curran Associates, Inc., 2015.

\bibitem[Mi et~al.(2021)Mi, Zhao, Nash, Jin, Gao, Sun, Schmid, Shavit, Chai,
  and Anguelov]{Mi2021}
Lu~Mi, Hang Zhao, Charlie Nash, Xiaohan Jin, Jiyang Gao, Chen Sun, Cordelia
  Schmid, Nir Shavit, Yuning Chai, and Dragomir Anguelov.
\newblock Hdmapgen: A hierarchical graph generative model of high definition
  maps.
\newblock In \emph{Proceedings of the IEEE/CVF Conference on Computer Vision
  and Pattern Recognition (CVPR)}, pp.\  4227--4236, June 2021.

\bibitem[Niu et~al.(2020)Niu, Song, Song, Zhao, Grover, and Ermon]{Niu2020}
Chenhao Niu, Yang Song, Jiaming Song, Shengjia Zhao, Aditya Grover, and Stefano
  Ermon.
\newblock Permutation invariant graph generation via score-based generative
  modeling.
\newblock In Silvia Chiappa and Roberto Calandra (eds.), \emph{Proceedings of
  the Twenty Third International Conference on Artificial Intelligence and
  Statistics}, volume 108 of \emph{Proceedings of Machine Learning Research},
  pp.\  4474--4484. PMLR, 26--28 Aug 2020.
\newblock URL \url{http://proceedings.mlr.press/v108/niu20a.html}.

\bibitem[Podda \& Bacciu(2021)Podda and Bacciu]{podda2021graphgenredux}
Marco Podda and Davide Bacciu.
\newblock Graphgen-redux: a fast and lightweight recurrent model for labeled
  graph generation, 2021.

\bibitem[Schwenk(1973)]{Schwenk73}
Allen~J. Schwenk.
\newblock Almost all trees are cospectral.
\newblock In Frank Harary (ed.), \emph{New Directions in the Theory of Graphs},
  pp.\  275--307. Academic Press, 1973.

\bibitem[Sutherland et~al.(2017)Sutherland, Tung, Strathmann, De, Ramdas,
  Smola, and Gretton]{Sutherland17}
Danica~J. Sutherland, Hsiao-Yu Tung, Heiko Strathmann, Soumyajit De, Aaditya
  Ramdas, Alex Smola, and Arthur Gretton.
\newblock Generative models and model criticism via optimized maximum mean
  discrepancy.
\newblock In \emph{International Conference on Learning Representations}, 2017.

\bibitem[Thompson et~al.(2022)Thompson, Knyazev, Ghalebi, Kim, and
  Taylor]{thompson2022on}
Rylee Thompson, Boris Knyazev, Elahe Ghalebi, Jungtaek Kim, and Graham~W.
  Taylor.
\newblock On evaluation metrics for graph generative models.
\newblock In \emph{International Conference on Learning Representations}, 2022.

\bibitem[van Dam \& Haemers(2003)van Dam and Haemers]{Dam03}
Edwing~R. van Dam and Willem~H. Haemers.
\newblock Which graphs are determined by their spectrum?
\newblock \emph{Linear Algebra and its Applications}, 373:\penalty0 241--272,
  2003.

\bibitem[Watts \& Strogatz(1998)Watts and Strogatz]{Watts98}
Duncan~J. Watts and Steven~H. Strogatz.
\newblock Collective dynamics of `small-world' networks.
\newblock \emph{Nature}, 393\penalty0 (6684):\penalty0 440--442, 1998.

\bibitem[You et~al.(2018)You, Ying, Ren, Hamilton, and Leskovec]{You18}
Jiaxuan You, Rex Ying, Xiang Ren, William Hamilton, and Jure Leskovec.
\newblock {G}raph{RNN}: Generating realistic graphs with deep auto-regressive
  models.
\newblock In \emph{Proceedings of the 35th International Conference on Machine
  Learning}, volume~80 of \emph{Proceedings of Machine Learning Research}, pp.\
   5708--5717. PMLR, 2018.

\bibitem[Zeng et~al.(2009)Zeng, Tung, Wang, Feng, and Zhou]{Zheng09}
Zhiping Zeng, Anthony K.~H. Tung, Jianyong Wang, Jianhua Feng, and Lizhu Zhou.
\newblock Comparing stars: On approximating graph edit distance.
\newblock \emph{Proceedings of the VLDB Endowment}, 2\penalty0 (1):\penalty0
  25--36, August 2009.
\newblock \doi{10.14778/1687627.1687631}.

\bibitem[Zhang et~al.(2021)Zhang, Zhao, Qin, Pfoser, and Ling]{Zhang2021}
Liming Zhang, Liang Zhao, Shan Qin, Dieter Pfoser, and Chen Ling.
\newblock Tg-gan: Continuous-time temporal graph deep generative models with
  time-validity constraints.
\newblock In \emph{Proceedings of the Web Conference 2021}, WWW '21, pp.\
  2104–2116. Association for Computing Machinery, 2021.

\end{thebibliography}
\bibliographystyle{iclr2022_conference}

\appendix
\section{Appendix}

The following sections provide additional details about the issues with
existing methods. We also show additional plots from our ranking and
perturbation experiments.

\subsection{Kernels based on total variation distance}\label{sec:Total variation distance}

Previous work used kernels based on the total variation distance in
order to compare evaluation functions via MMD. The choice of this
distance, however, requires subtle changes in the selection of kernels
for MMD---it turns out that the usual RBF kernel must \emph{not} be used
here!

We briefly recapitulate the definition of the total variation distance
before explaining its use in the kernel context:
given two finite-dimensional real-valued histograms $X := \{x_1, \dots,
x_n\}$ and $Y := \{y_1, \dots, y_n\}$, their \emph{total variation
distance} is defined as 
\begin{equation}
  \metric{TV}(X, Y) := \frac{1}{2}\sum_{i=1}^{n}\left|x_i - y_i\right|.
  \label{eq:TV}
\end{equation}
This distance induces a metric space that is not \emph{flat}. In other
words, the metric space induced by \autoref{eq:TV} has non-zero
curvature~(a fact that precludes certain kernels from being used
together with \autoref{eq:TV}). We can formalise this by showing that
the induced metric space \emph{cannot} have a bound on its curvature.
\begin{theorem}
  The metric space $X_{\mathrm{TV}}$ induced by the total variation distance
  between two histograms is not in $\CAT(k)$ for $k > 0$, where $\CAT(k)$ refers to
  the category of metric spaces with curvature bounded from above
  by~$k$~\citep{Gromov87}.
\end{theorem}
\begin{proof}
  Let $x_1 = (1, 0, \dots, 0)$ and $x_2 = (0, 1, 0, \dots, 0)$. There
  are at least two geodesics---shortest paths---of the same length, one that first
  decreases the first coordinate and subsequently increases the second
  one, whereas for the second geodesic this order is switched. More
  precisely, the first geodesic proceeds from $x_1$ to $x_1
  - (\epsilon, 0, \dots, 0)$ for an infinitesimal $\epsilon > 0$,
  until $x_0 = (0, \dots, 0)$ has been reached. Following this, the geodesic
  continues from $x_0$ to $x_0 + (0, \epsilon, 0, \dots, 0)$ in
  infinitesimal steps until $x_2$ has been reached. The order of these
  two operations can be switched, such that the geodesic goes from
  $x_1$ to $(1, 1, 0, \dots, 0)$, from which it finally continues to
  $x_2$. Both of these geodesics have a length of $1$.
  Since geodesics in a $\CAT(k)$ space for $k > 0$ are
  \emph{unique}~\citep[Proposition~2.11, p.~23]{Bridson99},
  $X_{\mathrm{TV}}$ is not in $\CAT(k)$ for $k >0$.
\end{proof}
Since every $\CAT(k)$ space is also a $\CAT(l)$ space for all $l > k$,
this theorem has the consequence that $X_{\mathrm{TV}}$ cannot be
a $\CAT(0)$ space. Moreover, as every \emph{flat} metric space is
in particular a $\CAT(0)$ space, $X_{\mathrm{TV}}$ is not flat.
According to Theorem~1 of \citet{Feragen15}, the associated geodesic
Gaussian kernel, i.e.\ the kernel that we obtain by writing
\begin{equation}
  \kernel(x, y) := \exp\mleft(-\frac{\metric{TV}(X, Y)^2}{2\sigma^2}\mright),
\end{equation}
is \emph{not} positive definite and should therefore not be used with
MMD.
One potential fix for this specific distance involves using the
Laplacian kernel, i.e.\
\begin{equation}
  \kernel(x, y) := \exp\mleft(-\lambda\metric{TV}\mleft(X, Y\mright)\mright).
\end{equation}
The subtle difference between these kernel functions---only
an exponent is being changed---demonstrate that care must be taken when
selecting kernels for use with MMD.

\subsection{Experimental setup}\label{sec:Experimental Setup}

We can analyse the desiderata outlined above using an experimental
setting. In the following, we will assess expressivity, robustness, and
efficiency for a set of common perturbations, i.e.\
\begin{inparaenum}[(i)]
  \item random edge insertions, \label{item:Insertion} 
  \item random edge deletions, \label{item:Deletion}
  \item random rewiring operations, i.e.\ `swapping' edges\label{item:Rewiring}, and
  \item random node additions\label{item:Node addition}.
\end{inparaenum}
For each perturbation type, we will investigate how the metric changes for
an ever-increasing degree of perturbation, where each perturbation is
parametrized by at least one parameter. When removing edges, for instance,
this is the probability of removing an edge in the graph.  Thus for a graph
with $100$ edges and $p_\text{remove}=0.1$ we would expect on average $90$ edges to
remain in the graph.  Similar parametrizations apply for the other
perturbations, for anmore detailed description we refer to Appendix~\ref{sec:Details graph perturbations}.
All these operations are inherently small-scale~(though
not necessarily localised to specific regions within a graph), but turn
into large-scale perturbations of a graph depending on the strength of
the perturbation performed. 
We performed these perturbations using our own \texttt{Python}-based
framework for graph generative model comparison and evaluation. Our
framework additionally permits the simple integration of additional descriptor
functions and evaluators.

\subsection{Details on graph perturbations}\label{sec:Details graph perturbations}
We describe all graph perturbations used in this work on the example of a single graph $G:= (V, E)$ where $V$
refers to the vertices of the graph and $E$ to the edges.

\paragraph{Add Edges}
For each $v_i, v_j \in V$ with $v_i \neq v_j$ a sample from a Bernoulli distribution $x_{ij} \sim
\text{Ber}(p_\text{add})$ is drawn. Samples for which $x_{ij}=1$ are added to the list of edges such that $E' = E \cup
\{(v_i, v_j) \mid x_{ij} = 1\}$.

\paragraph{Remove Edges}
For each $e_i \in E$, a sample from a Bernoulli distribution $x_i \sim \text{Ber}(p_\text{remove})$ is drawn, and samples
with $x_i = 1$ are removed from the edge list, such that $E' = E \cap \{e_i \mid x_i \neq 1\}$.

\paragraph{Rewire Edges}
For each $e_i \in E$, a sample from a Bernoulli distribution $x_i \sim \text{Ber}(p_\text{rewire})$ is drawn, and samples
with $x_i = 1$ are rewired.  For rewiring a further random variable $y_i \sim \text{Ber}(0.5)$ is drawn which determines
which node $e_i[y_i]$ of the edge $e_i$ is kept.  The node to which the edge is connected is chosen uniformly from the
set of vertices $v_i \in V$, where $v_i \notin e_i$, i.e. avoiding self-loops and reconnecting the original edge.
Finally the original edge is removed and the new edge $e'_i = (e_i[y_i], v_i)$ is added to the graph $E' = E \cap \{e_i \mid
x_i \neq 1\} \cup \{e'_i \mid x_i = 1\}$.

\paragraph{Add Connected Node}
We define a set of vertices to be added $V^* = \{v_i \mid |V| < i \leq |V|+n\}$, where $n$ represents the number of nodes
to add.  For each $v_i \in V$ and $v_j \in V^*$ we draw a sample from a Bernoulli distribution $x_{ij} \sim
\text{Ber}(p_\text{connect\_node})$ and an edge between $v_i$ and $v_j$ to the graph if $x_{ij} = 1$. Thus $E' = E \cup
\{(v_i, v_j) \mid v_i \in V, v_j \in V^*, x_{ij} = 1\}$.

\subsection{Implementation details}\label{sec:Implementation details}

We used the official implementations of GraphRNN, GRAN and Graph Score
Matching in our experiments. GraphRNN and GRAN both have an MIT License,
and Graph Score Matching is licensed under GNU General Public License
v3.0.
Our code is available at~(\url{https:/www.github.com/BorgwardtLab/ggme}) under a BSD 3-Clause license.  

\paragraph{Compute resources.}
All the jobs were run on our internal cluster, comprising 64 physical
cores~(\texttt{Intel(R) Xeon(R) CPU E5-2620 v4 @ 2.10GHz}) with 8
GeForce GTX 1080 GPUs. We stress that the main component of this paper,
i.e.\ the evaluation itself, do \emph{not} necessarily
require a cluster environment. The cluster was chosen because individual
generative models had to be trained in order to obtain generated graphs,
which we could subsequently analyse and rank.

\subsection{Speed up trick}\label{sec:Speed up}
%
As the combination of kernels and hyperparameters yielded hundreds of
combinations, it is worth mentioning a worthwhile speedup trick to
reduce the complexity of assessing the Gaussian and Laplacian kernel
combinations. Considering they have a shared intermediate value, namely
the Euclidean distance, it is possible to return intermediate values in
the MMD computation ($K_{XX}$, $K_{YY}$, and $K_{XY}$, prior to
exponentiation, potential squaring, and scaling by $\sigma$. Storing
these intermediate results allows one to rapidly iterate over a grid of
values for $\sigma$ without needed to recalculate MMD, leading to
a worthwhile speedup.

\subsection{Experimental results}\label{sec:Full perturbation results}

We now provide the results presented in the paper across all the datasets.

\begin{figure}[h!]
    \centering
    \begin{subfigure}[b]{0.24\textwidth}
    \includegraphics[width=\textwidth]{Figures/kernel_choice/hyperparameter_heatmap_BarabasiAlbertGraphs_clustering_gaussian.pdf}
    \caption{Barab\'{a}si-Albert Graphs}
    \end{subfigure}
    \hfill
    \begin{subfigure}[b]{0.24\textwidth}
    \centering
    \includegraphics[width=\textwidth]{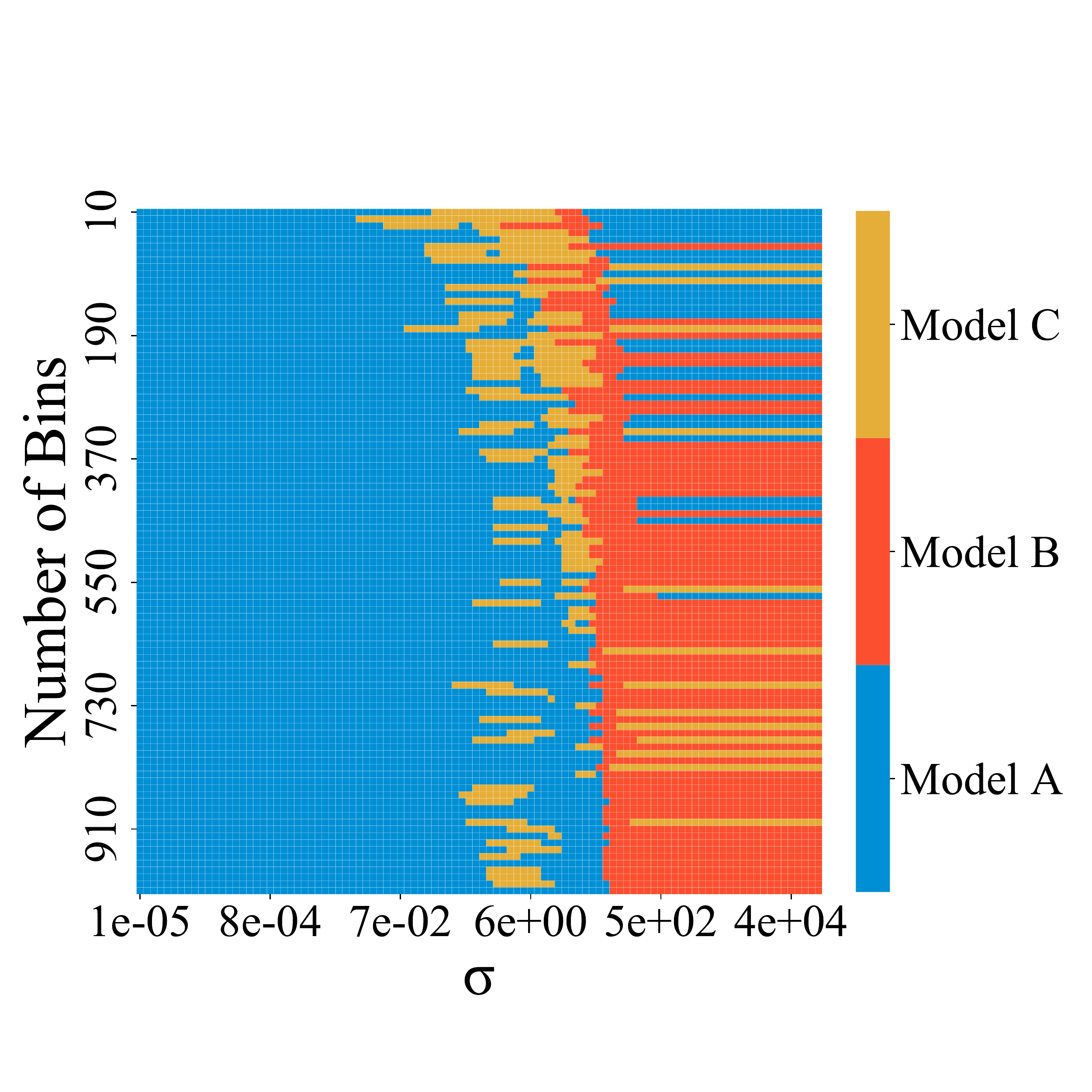}
    \caption{Community Graphs}
    \end{subfigure}
    \hfill
    \begin{subfigure}[b]{0.24\textwidth}
    \centering
    \includegraphics[width=\textwidth]{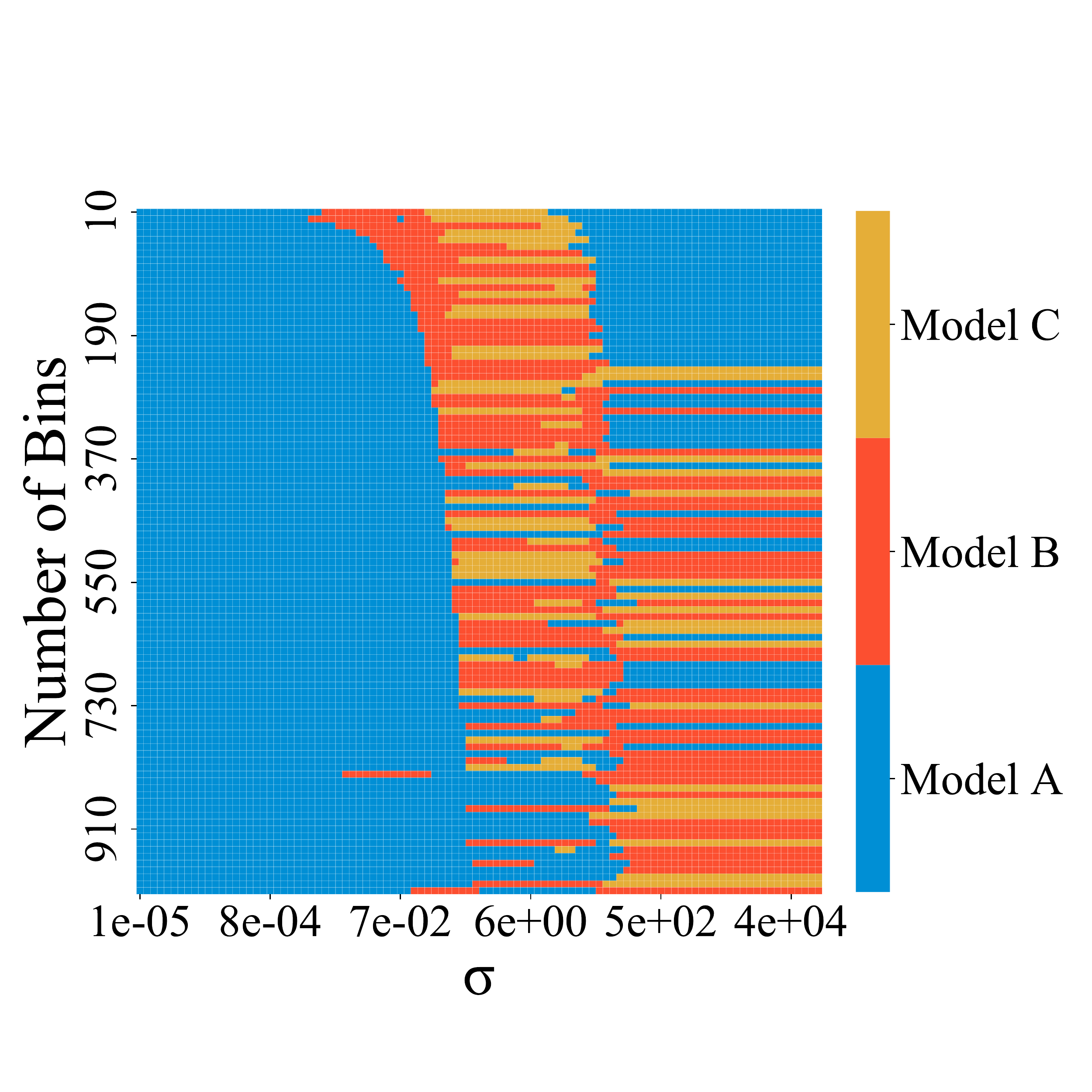}
    \caption{Erd\"{o}s-R\'{e}nyi Graphs}
    \end{subfigure}
    \hfill
    \begin{subfigure}[b]{0.24\textwidth}
    \centering
    \includegraphics[width=\textwidth]{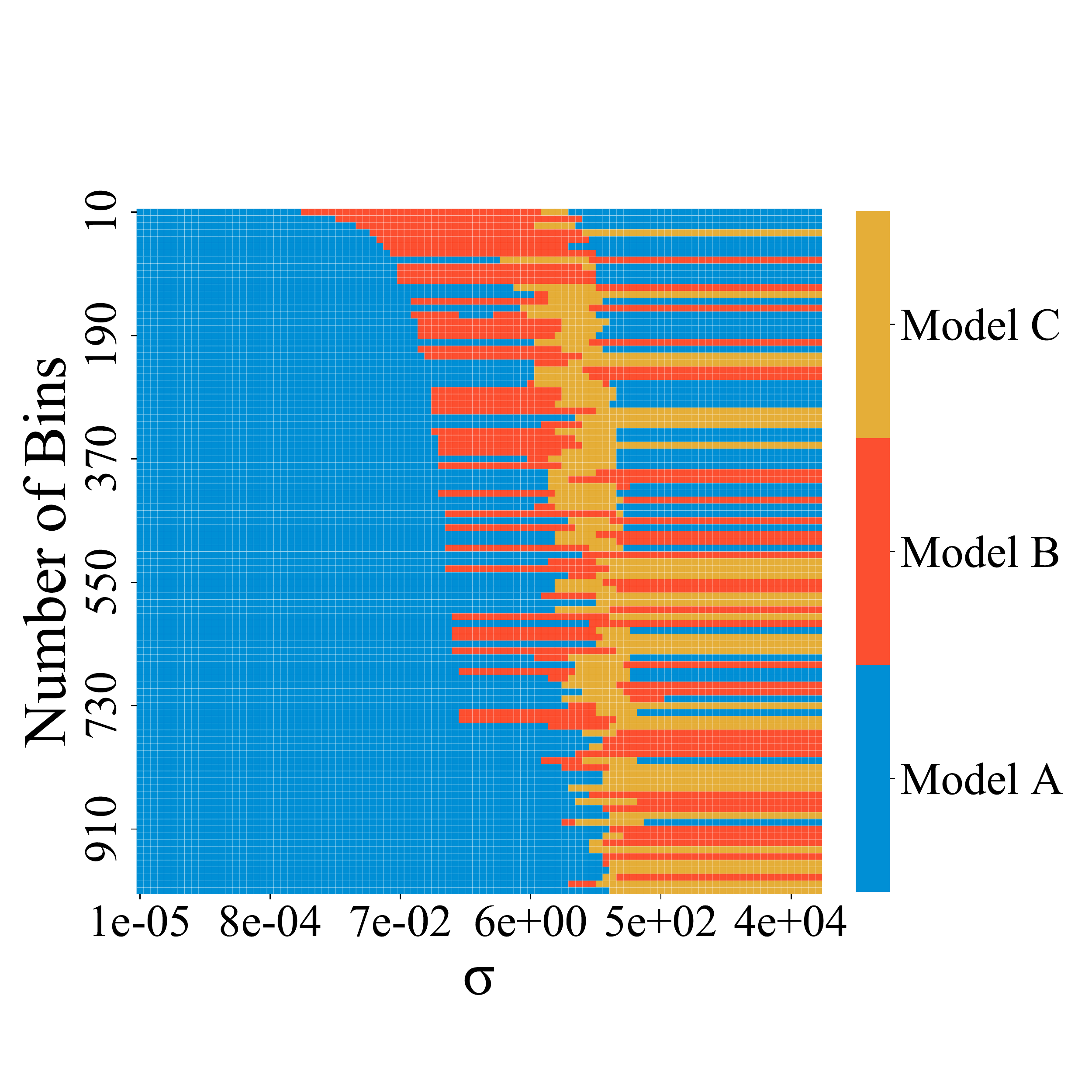}
    \caption{Watts-Strogatz Graphs}
    \end{subfigure}
    \caption{A heatmap of which model (from A, B, and C) ranks first in terms of MMD across different hyperparameter combinations. This uses the clustering coefficient descriptor function and the RBF kernel, where the number of bins is a hyperparameter of the descriptor function, and $\sigma$ is the hyperparameter in the kernel.}
    \end{figure}

\begin{figure}
    \centering
    \includegraphics[width=\textwidth]{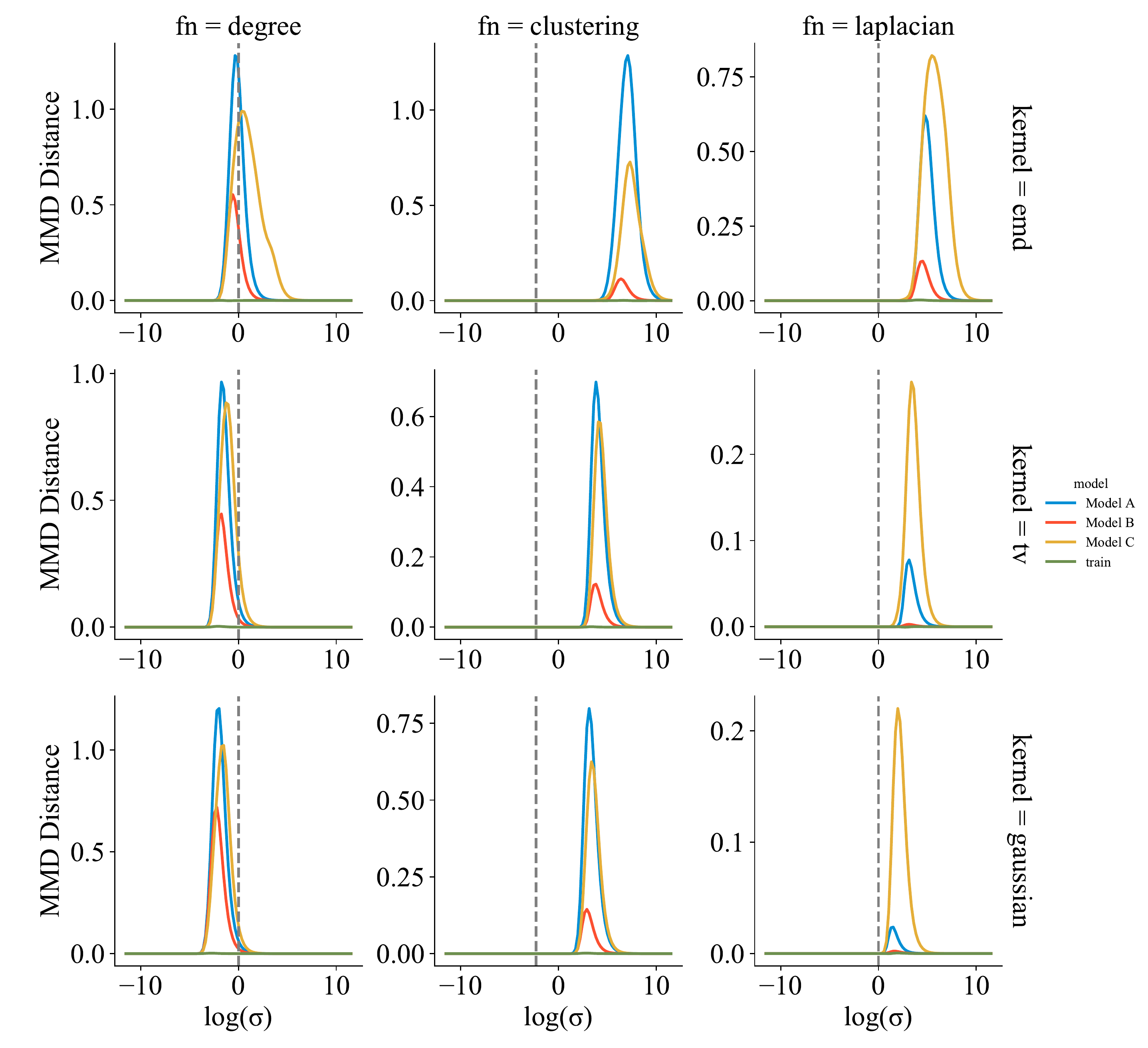}
    \caption{\textbf{Barab\'{a}si-Albert Graphs}. MMD calculated between the test graphs and predictions from three recent graph generative models (A, B, C) over a range of values of $\sigma$ on the Barab\'{a}si-Albert Graphs dataset. Additionally, the MMD distance between the test graphs and training graphs is provided to give a meaningful sense of scale to the metric. It provides an idea of what value of MMD signifies an indistinguishable difference between the two distributions.}
    \label{fig:kernel_vs_fn_BA}
\end{figure}

\begin{figure}
    \centering
    \includegraphics[width=\textwidth]{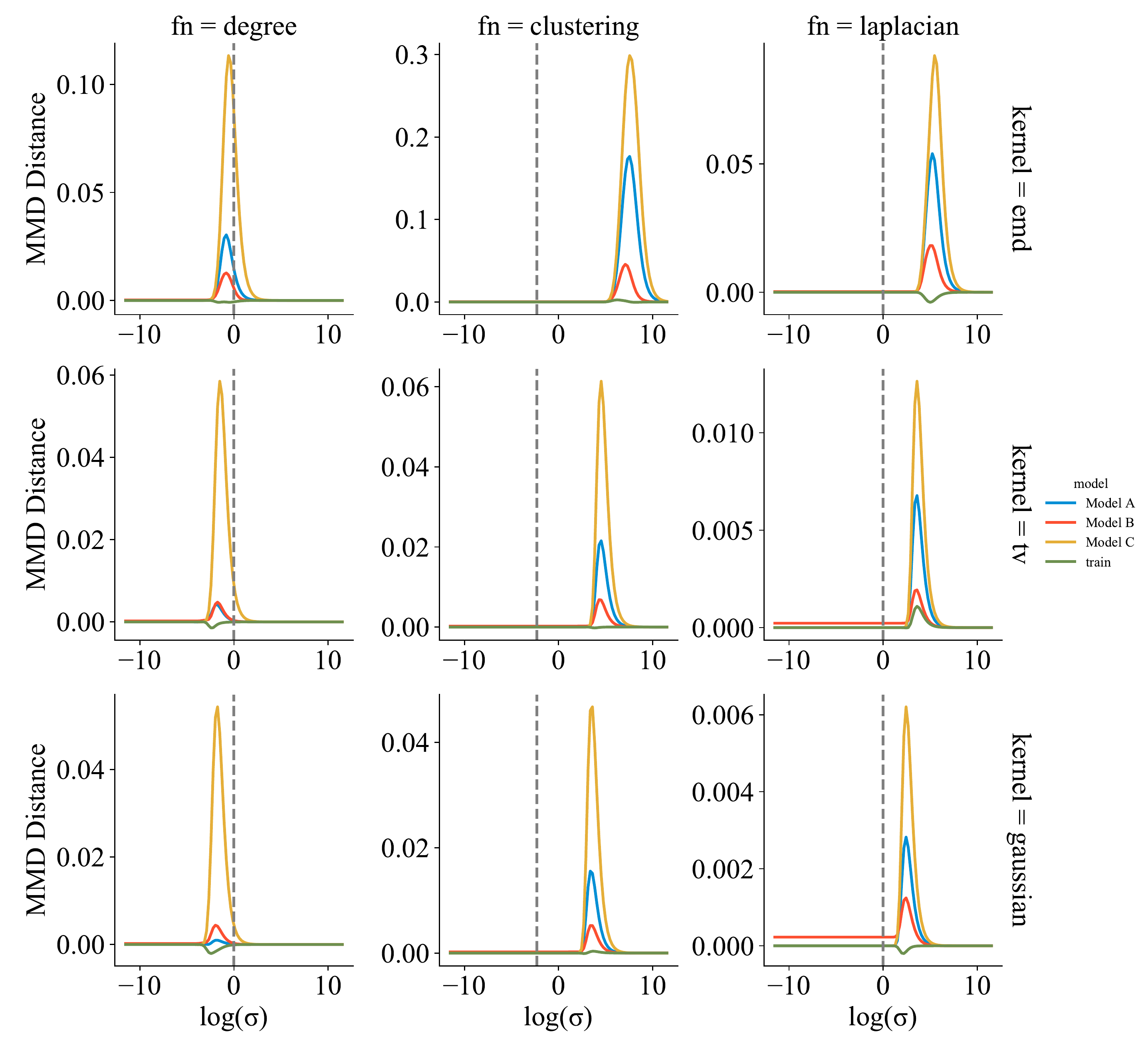}
    \caption{\textbf{Community Graphs}. MMD calculated between the test graphs and predictions from three recent graph generative models (A, B, C) over a range of values of $\sigma$ on the Community Graphs dataset. Additionally, the MMD distance between the test graphs and training graphs is provided to give a meaningful sense of scale to the metric. It provides an idea of what value of MMD signifies an indistinguishable difference between the two distributions. }
    \label{fig:kernel_vs_fn_CG}
\end{figure}

\begin{figure}
    \centering
    \includegraphics[width=\textwidth]{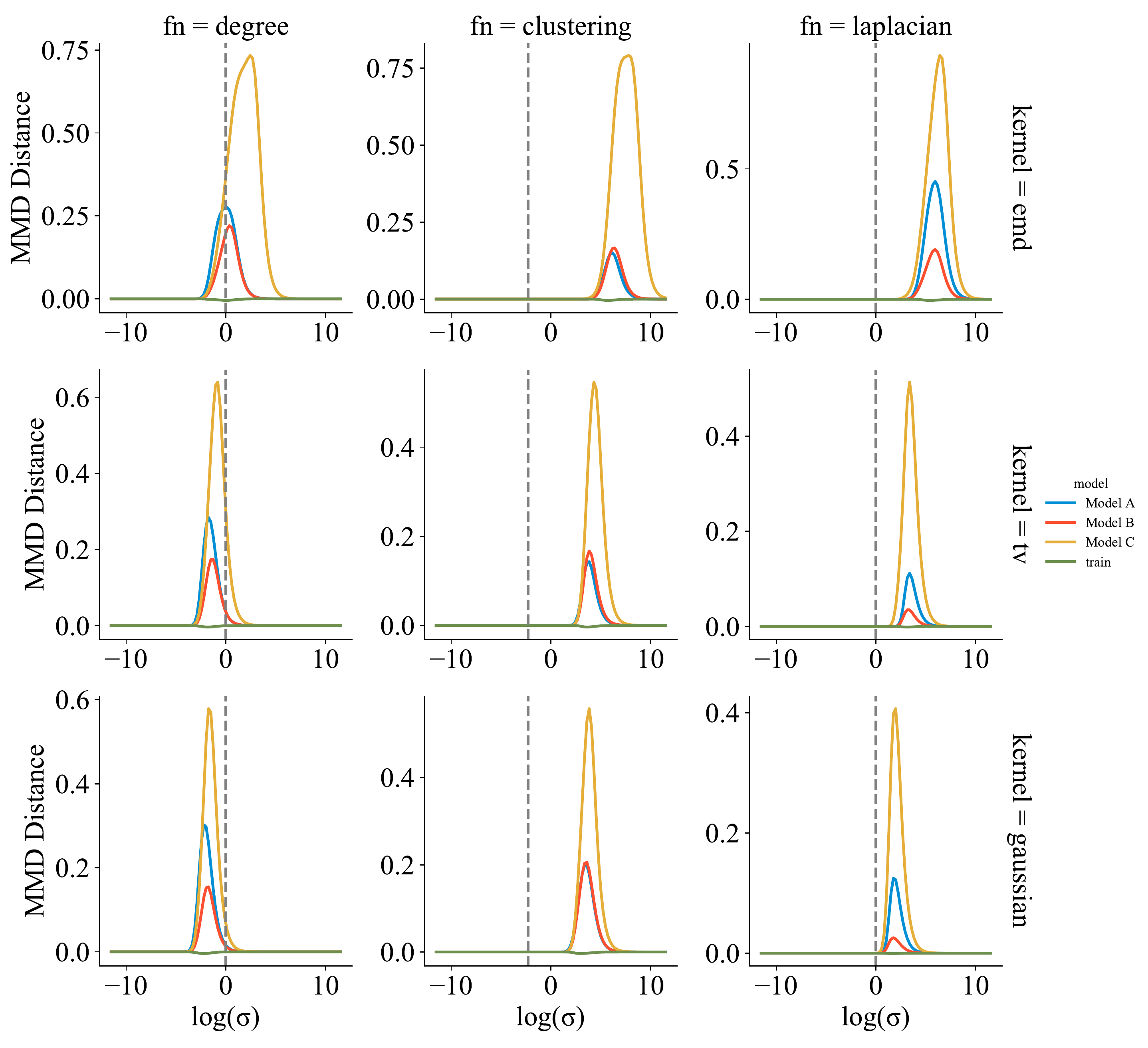}
    \caption{\textbf{Erd\"{o}s-R\'{e}nyi Graphs}. MMD calculated between the test graphs and predictions from three recent graph generative models (A, B, C) over a range of values of $\sigma$ on the Erd\"{o}s-R\'{e}nyi Graphs dataset. Additionally, the MMD distance between the test graphs and training graphs is provided to give a meaningful sense of scale to the metric. It provides an idea of what value of MMD signifies an indistinguishable difference between the two distributions. }
    \label{fig:kernel_vs_fn_ER}
\end{figure}

\begin{figure}
    \centering
    \includegraphics[width=\textwidth]{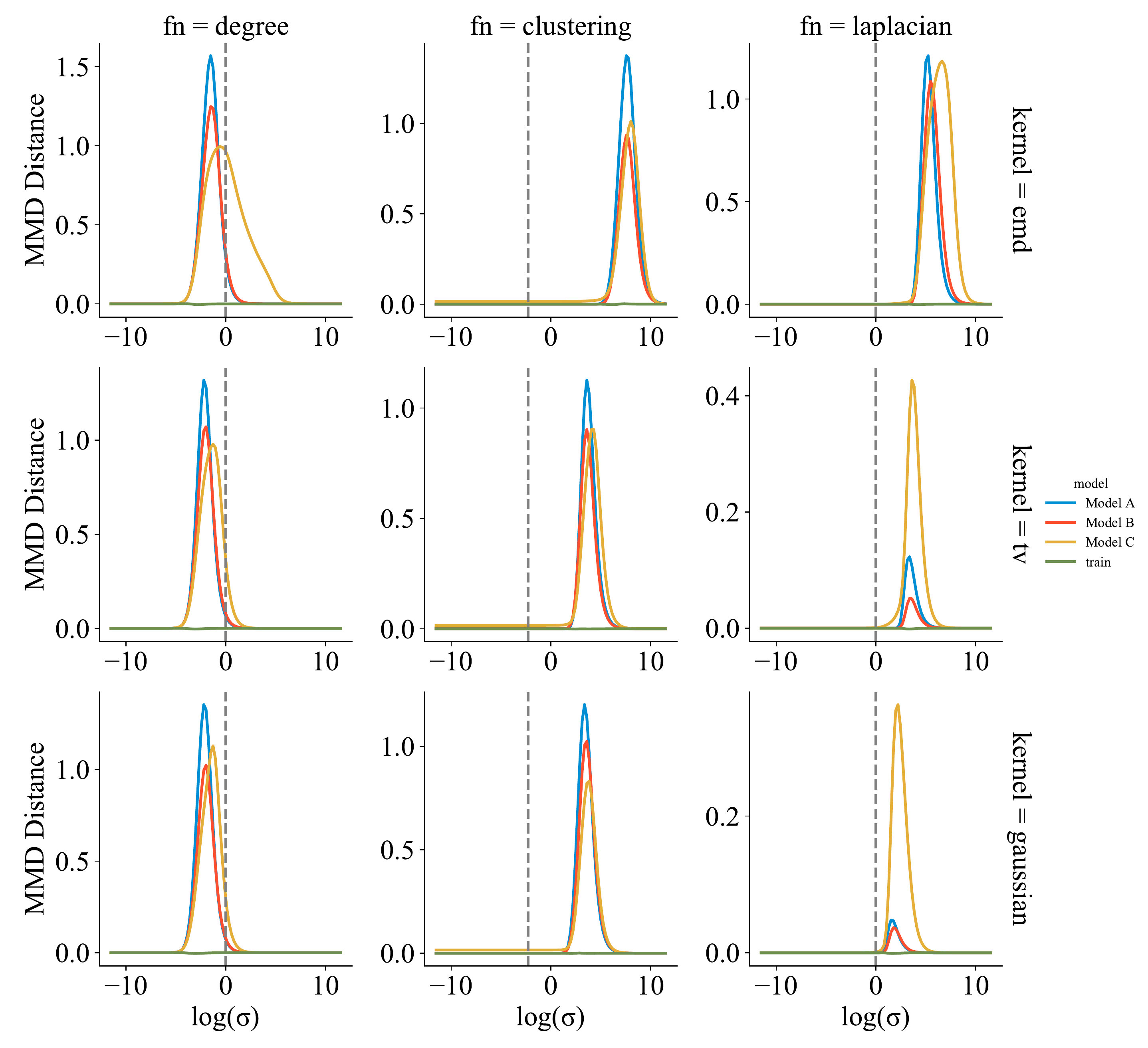}
    \caption{\textbf{Watts-Strogatz Graphs}. MMD calculated between the test graphs and predictions from three recent graph generative models (A, B, C) over a range of values of $\sigma$ on the Watts-Strogatz Graphs dataset. Additionally, the MMD distance between the test graphs and training graphs is provided to give a meaningful sense of scale to the metric. It provides an idea of what value of MMD signifies an indistinguishable difference between the two distributions.}
    \label{fig:kernel_vs_fn_WS}
\end{figure}



\begin{figure}[tbp]
    \centering
    \resizebox{\textwidth}{!}{%
        \includegraphics[]{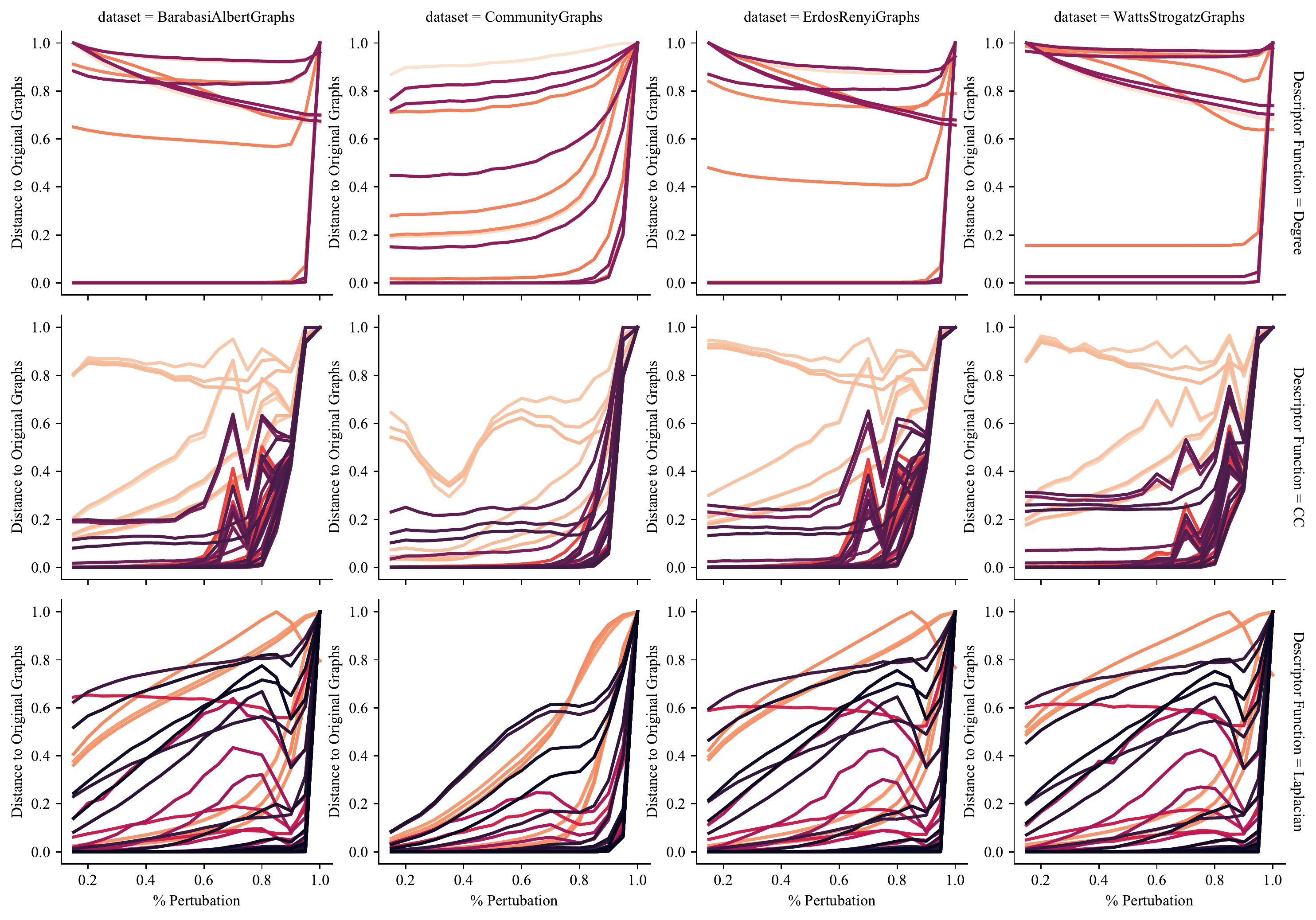}
        }%
    
    \caption{%
      \textbf{Perturbation: adding edges}. This figure shows the full results across datasets, descriptor functions and parameters when the perturbation is adding edges to the graphs. At each level of perturbation, the distance of the perturbed graphs is calculated to the original graph distribution using the specified evaluator function. Each line represents a different parameter combination. An ideal evaluator function would monotonically increase as the degree of perturbation increases.}
    
    \label{fig:Full perturbation results: MMD, Adding Edges}
\end{figure}


\begin{figure}[tbp]
    \centering
    \resizebox{\textwidth}{!}{%
        \includegraphics[]{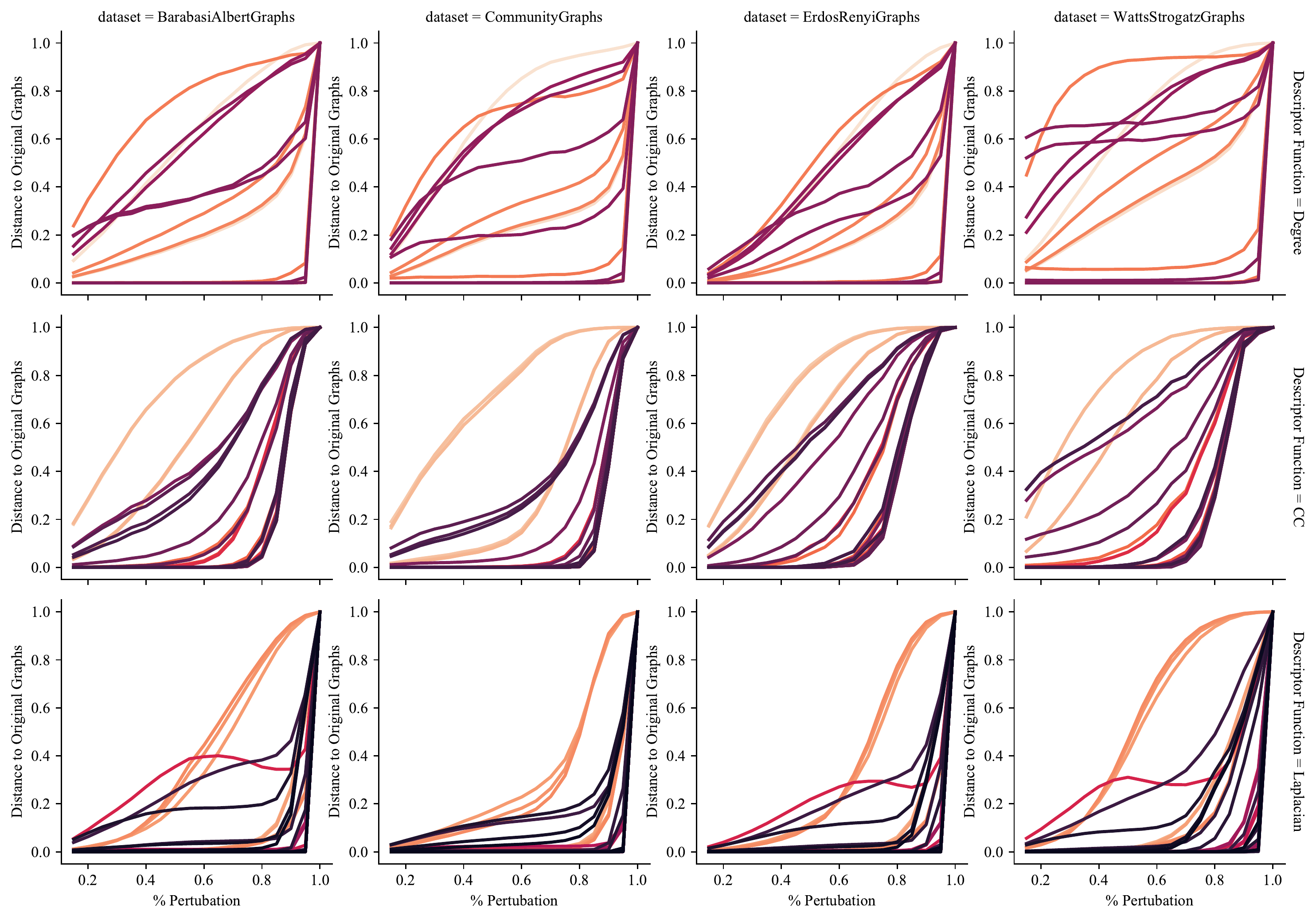}
        }%
    
    \caption{%
      \textbf{Perturbation: removing edges}. This figure shows the full results across datasets, descriptor functions and parameters when the perturbation is removing edges from the graphs. At each level of perturbation, the distance of the perturbed graphs is calculated to the original graph distribution using the specified evaluator function. Each line represents a different parameter combination. An ideal evaluator function would monotonically increase as the degree of perturbation increases.
      }
    
    \label{fig:Full perturbation results: MMD, Removing Edges}
\end{figure}


\begin{figure}[tbp]
    \centering
    \resizebox{\textwidth}{!}{%
        \includegraphics[]{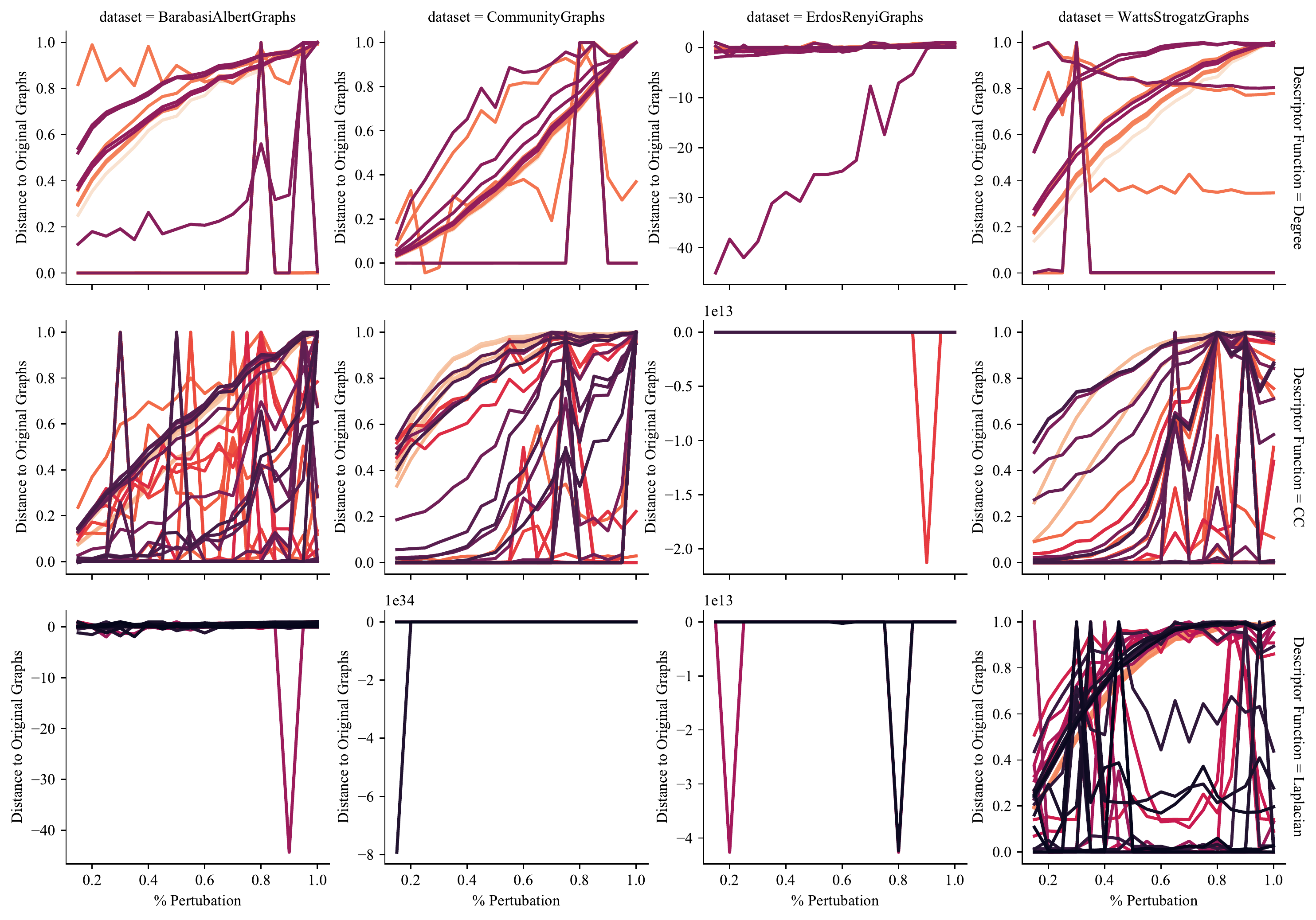}
        }%
    
    \caption{%
      \textbf{Perturbation: rewiring edges}. This figure shows the full results across datasets, descriptor functions and parameters when the perturbation is rewiring edges in the graphs. At each level of perturbation, the distance of the perturbed graphs is calculated to the original graph distribution using the specified evaluator function. Each line represents a different parameter combination. An ideal evaluator function would monotonically increase as the degree of perturbation increases.
      }
    
    \label{fig:Full perturbation results: MMD, Rewiring Edges}
\end{figure}


\begin{figure}[tbp]
    \centering
    \resizebox{\textwidth}{!}{%
        \includegraphics[]{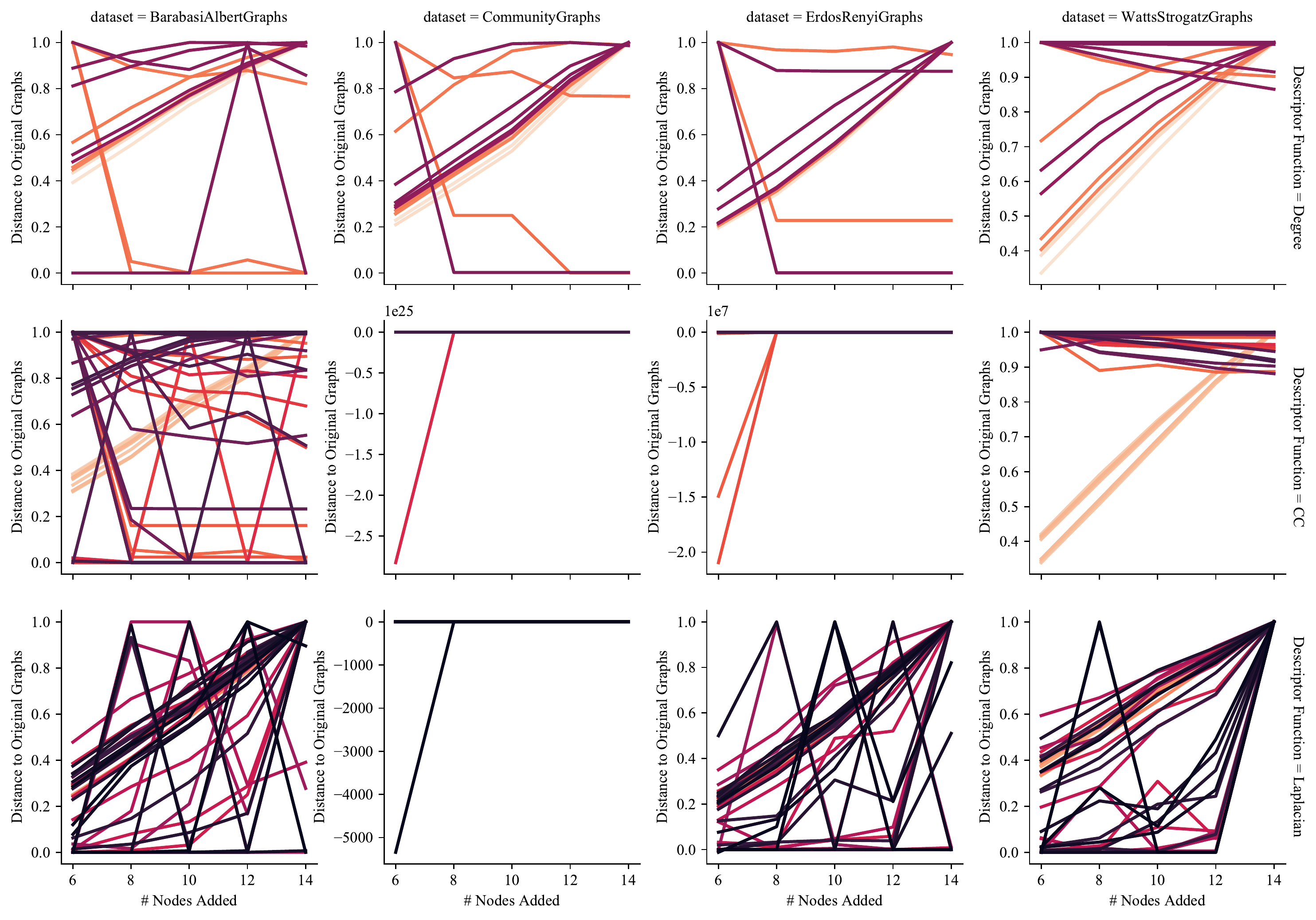}
        }%
    
    \caption{%
      \textbf{Perturbation: adding connected nodes}. This figure shows the full results across datasets, descriptor functions and parameters when the perturbation is adding connected nodes to the graphs (for each node that is added, there is a 15\% chance the node will be connected to any other node in the graph). At each level of perturbation, the distance of the perturbed graphs is calculated to the original graph distribution using the specified evaluator function. Each line represents a different parameter combination. An ideal evaluator function would monotonically increase as the degree of perturbation increases.
      }
    
    \label{fig:Full perturbation results: MMD, Add Connected Nodes}
\end{figure}

\clearpage
\subsection{\added{Alternative measures of correlation}}\label{sec:Alternative correlation measures}

\added{As a general recommendation, we chose to use the Pearson correlation coefficient to select a good kernel-hyperparameter combination, due to its simplicity and ease of interpretation. It reflects the behavior we expect from perturbations: as the graphs are increasingly perturbed, the distance to the original graphs should grow in a similar manner as well. However, there could be scenarios in which the distance to the original graphs should not grow linearly with the degree of perturbation, in which case the Pearson correlation coefficient would not be the best choice. We present here two plug-in alternatives for the Pearson correlation coefficient, namely the Spearman rank correlation coefficient, and mutual information, which can easily be integrated into our framework. We add one word of caution when using the mutual information, which is the fact that it does not capture the directionality of dependence. This is only an issue if there are scenarios in which the MMD distance \emph{decreases} as the degree of perturbation increases. Since we observed this in some of our datasets, it would not be the most appropriate to use in this specific case. We now  present the results from using these measures of dependence below.}

\begin{figure}[h!]
    \centering
    
    \begin{subfigure}[b]{0.23\textwidth}
        \includegraphics[height=4cm]{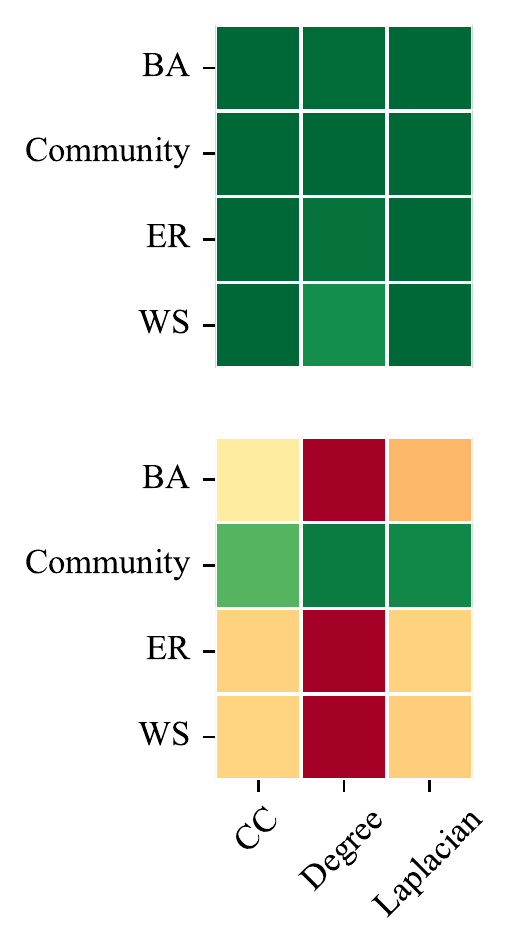}
        \caption{AddEdges}
    \end{subfigure}%
    \begin{subfigure}[b]{0.23\textwidth}
        \includegraphics[height=4cm]{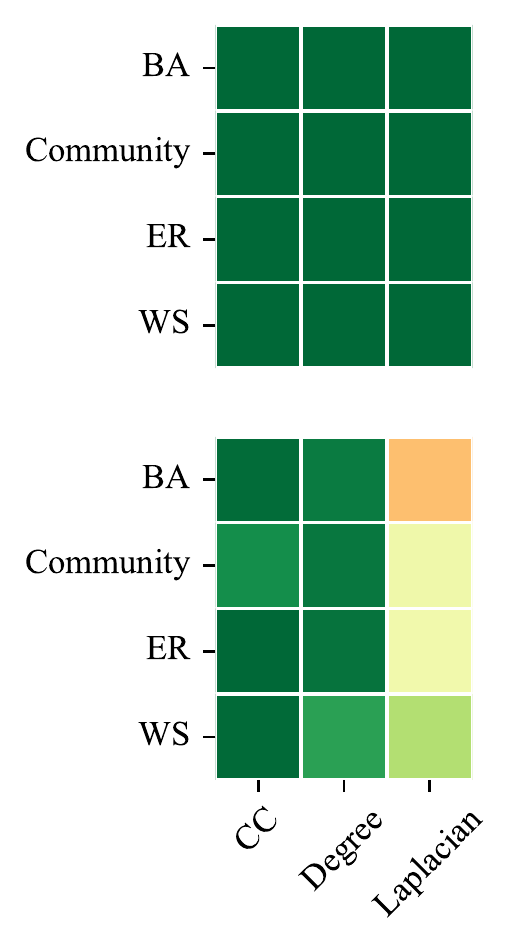}
        
        \caption{RemoveEdges}
    \end{subfigure}%
    \begin{subfigure}[b]{0.23\textwidth}
        \includegraphics[height=4cm]{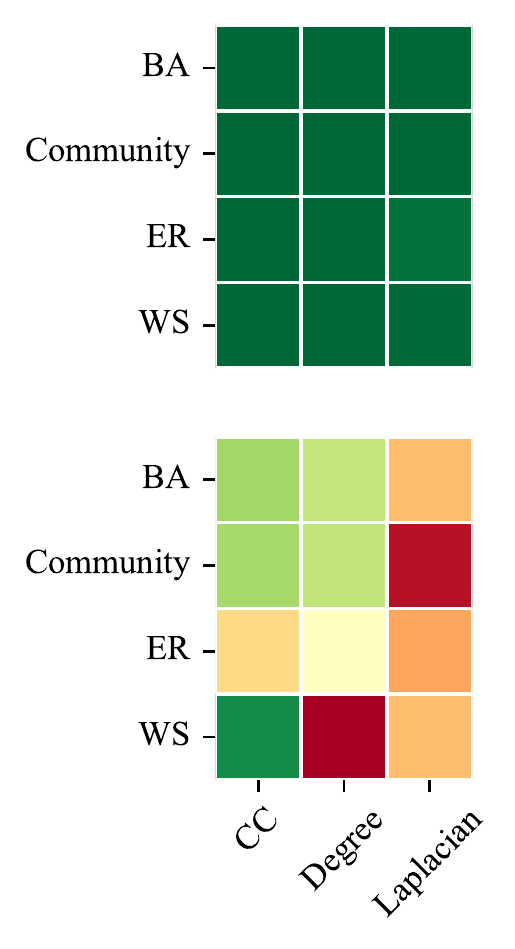}
        \caption{RewireEdges}
    \end{subfigure}
    \begin{subfigure}[b]{0.23\textwidth}
        \includegraphics[height=4cm]{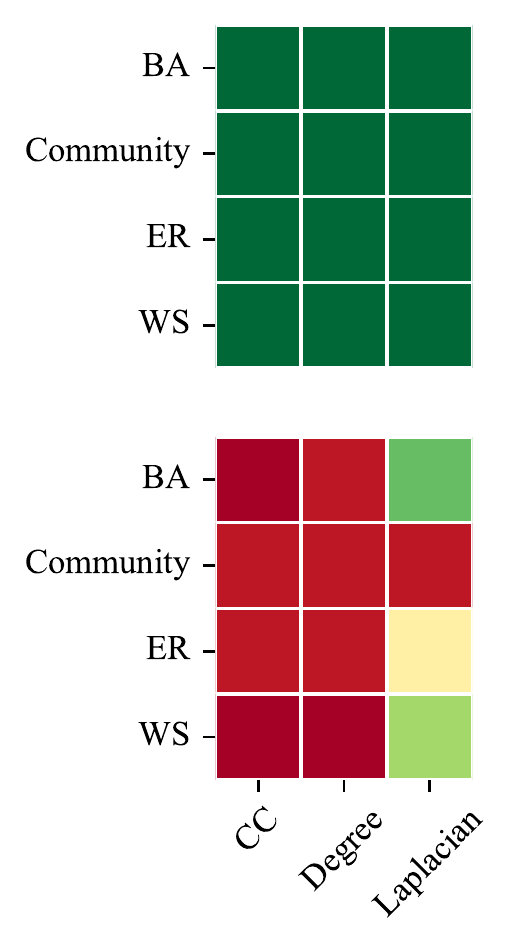}
        \caption{AddConnectedNodes}
    \end{subfigure}
    \begin{subfigure}[b]{0.03\textwidth}
        \includegraphics[height=4cm]{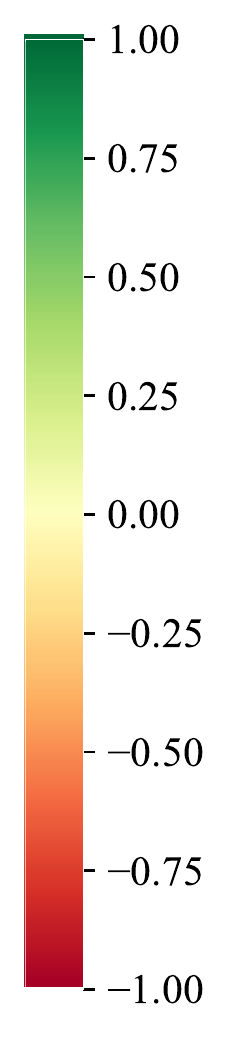}
        \caption*{}
    \end{subfigure}
    \caption{%
      \added{The Spearman rank correlation of MMD with
      the degree of perturbation in the graph, assessed for different
      descriptor functions and datasets (BA: Barab\'{a}si-Albert Graphs, ER: Erd\"{o}s-R\'{e}nyi Graphs, WS: Watts-Strogatz Graphs). For an ideal metric, the
      distance would increase as the degree of perturbation increases;
      resulting in a correlation close to 1. The upper row shows the
      \emph{best} kernel-parameter combination in terms of the correlation; the bottom row shows the \emph{worst}. As we can see, a proper kernel and parameter selection leads to strong correlation to the perturbation, but a bad choice  can lead to inverse correlation, highlighting the importance of a good kernel/parameter combination. }
    }
    \label{fig:Spearman Perturbation Heatmap}
\end{figure}

\begin{figure}[h!]
    \centering
    
    \begin{subfigure}[b]{0.23\textwidth}
        \includegraphics[height=4cm]{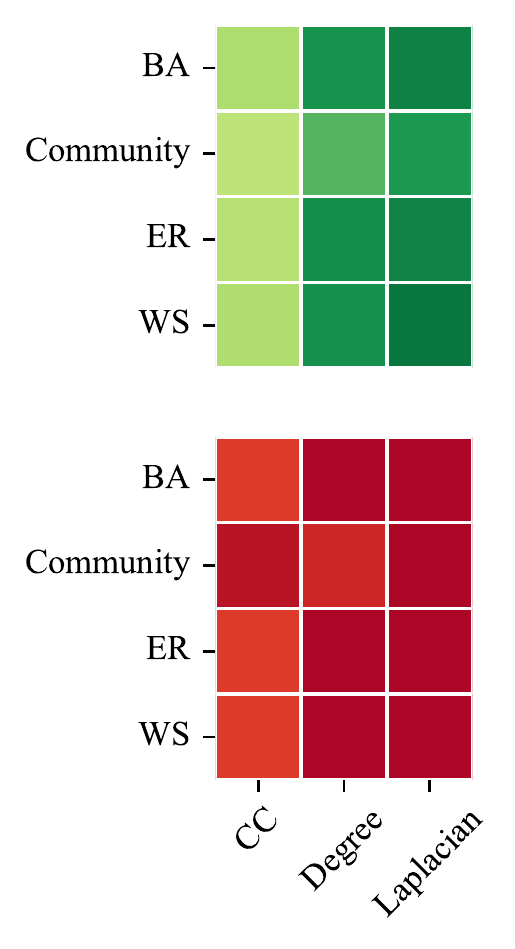}
        \caption{AddEdges}
    \end{subfigure}%
    \begin{subfigure}[b]{0.23\textwidth}
        \includegraphics[height=4cm]{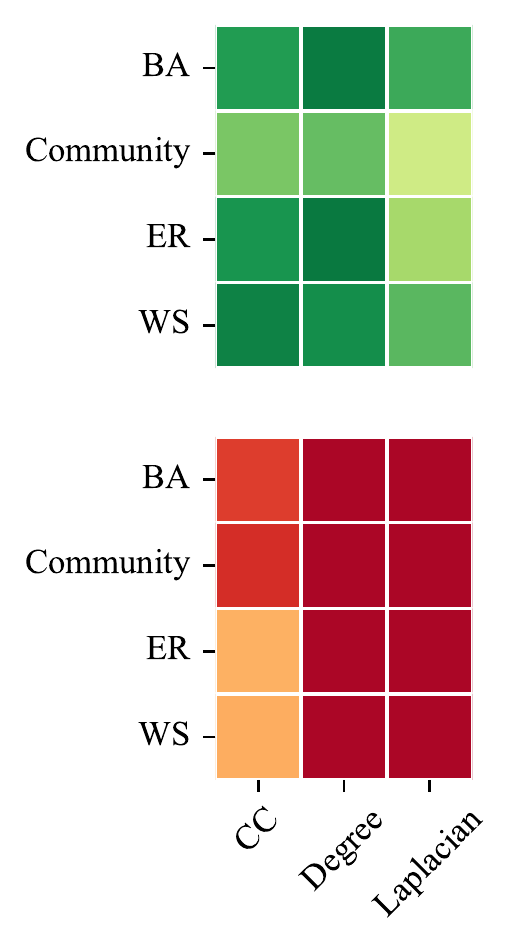}
        
        \caption{RemoveEdges}
    \end{subfigure}%
    \begin{subfigure}[b]{0.23\textwidth}
        \includegraphics[height=4cm]{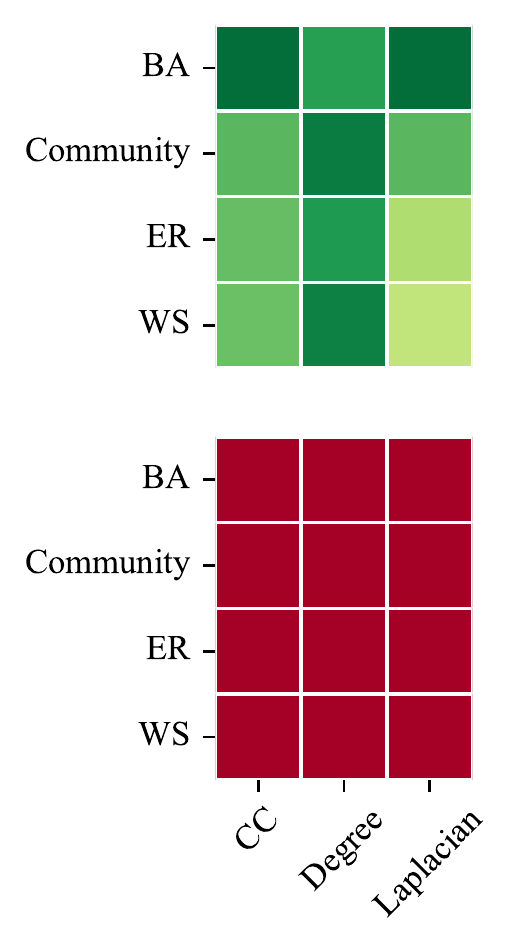}
        \caption{RewireEdges}
    \end{subfigure}
    \begin{subfigure}[b]{0.23\textwidth}
        \includegraphics[height=4cm]{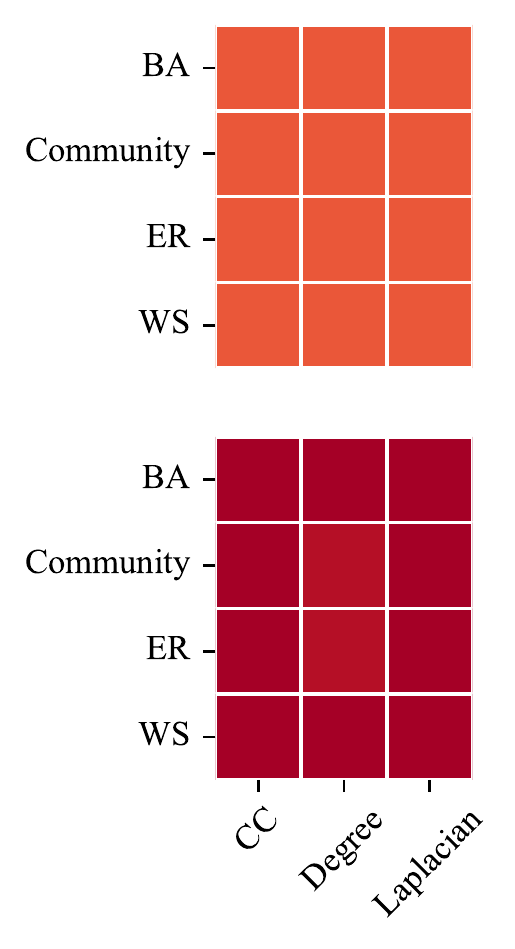}
        \caption{AddConnectedNodes}
    \end{subfigure}
    \begin{subfigure}[b]{0.03\textwidth}
        \includegraphics[height=4cm]{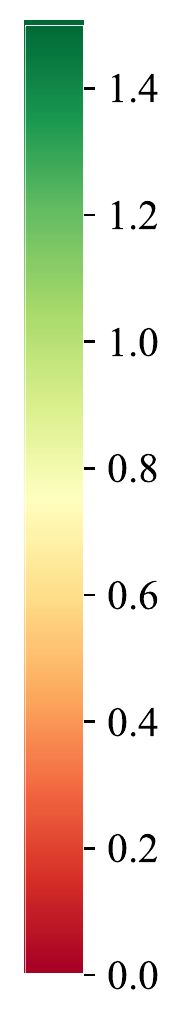}
        \caption*{}
    \end{subfigure}
    \caption{%
      \added{The mutual information of MMD with
      the degree of perturbation in the graph, assessed for different
      descriptor functions and datasets (BA: Barab\'{a}si-Albert Graphs, ER: Erd\"{o}s-R\'{e}nyi Graphs, WS: Watts-Strogatz Graphs). For an ideal metric, the
      distance would increase as the degree of perturbation increases;
      resulting in a high mutual information coefficient. The upper row shows the
      \emph{best} kernel-parameter combination in terms of the mutual information; the bottom row shows the \emph{worst}. As we can see, a proper kernel and parameter selection leads to strong dependence on the perturbation, but a bad choice  can lead to no mutual information, highlighting the importance of a good kernel/parameter combination. }
    }
    \label{fig:MI Perturbation Heatmap}
\end{figure}

\subsection{\added{Computational runtime of different kernels}\label{sec:kernel_runtime}}

\added{In the following, we assess the empirical CPU runtime of the linear kernel, RBF kernel, and the EMD-based kernel in the MMD calculation. As noted by ourselves and \citet{Liao19}, the EMD-based kernel is computationally expensive, hindering its suitability for use in graph generative model evaluation. We investigate this effect in more detail by doing a runtime comparison of the three kernels in a simulated environment. We generate two sets of Erd\"{o}s-R\'{e}nyi Graphs with the probability of an edge $p=0.3$, and then calculate the MMD distance between the two sets using the degree distribution as the descriptor function for a varying dataset size ($n_{\mathrm{graphs}}$), graph size ($n_\mathrm{nodes}$), and histogram bin size ($n_{\mathrm{bins}}$). As a default, we set $n_{\mathrm{graphs}}=100$, $n_\mathrm{nodes}=100$, and  $n_{\mathrm{bins}}=100$. We then change one variable at a time, iterating through values $\{100, 200, \ldots, 1000\}$ while keeping the other two variables fixed, and measured the time it took to calculate the MMD distance. We report the average runtime over ten repetitions to obtain more stable results.

Our results can be seen in Figure~\ref{fig:runtime}. At the default setting of 100 graphs, 100 nodes per graph, and 100 bins in the histogram, we observed that the EMD-based kernel took more than 50-140 times longer to compute compared to the RBF and linear kernels respectively. This difference was exacerbated when the size of the dataset ($n_{\mathrm{graphs}}$) increased as well as when the size of the graphs in the dataset ($n_{\textrm{nodes}})$ increased. For a dataset of 1,000 graphs, the EMD-based kernel took 27 minutes to run, whereas the linear kernel took 6 seconds and the RBF kernel took 21 seconds. Doing a simple extrapolation of the runtime of a dataset comprising 10,000 graphs, we estimate the EMD-based kernel would take 50 hours to run for a single hyperparameter combination, showcasing the computational limitation of this choice of kernel for larger dataset sizes. We observed a similar, albeit less extreme, increase in runtime when the size of the graphs increased.  While the linear time approximation of MMD \citep{Gretton12} could mitigate some of the runtime challenges, in general we recommend that practitioners use efficient kernels such as the linear kernel or RBF kernel for graph generative model evaluation.
}
\begin{figure}[tbp]
    \centering
    \includegraphics[width=\textwidth]{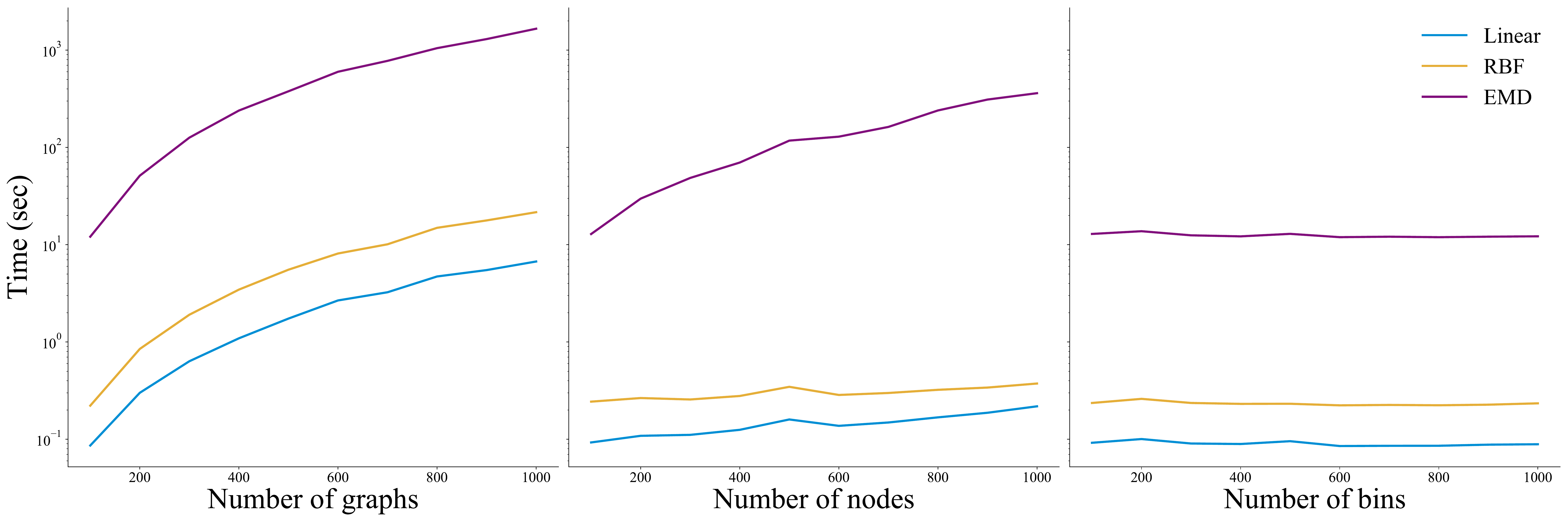}
    \caption{\added{CPU runtime comparison (lower is better) of the linear, RBF, and EMD-based kernels when evaluated on ER graphs and varying the number of graphs (the dataset size), the number of nodes (the size of each graph), and histogram bin size. Each plot varies a single parameter (on the $x$-axis), while keeping the other two fixed with values of 100. Runtimes are reported on a logarithmic scale.}}
    \label{fig:runtime}
\end{figure}

\end{document}